\documentclass[twoside,11pt]{article}

\usepackage[preprint]{jmlr2e}

\usepackage{amsmath}
\usepackage{amsfonts}
\usepackage{mathtools}
\usepackage{bbm,bm}
\usepackage{booktabs,multirow}
\usepackage{makecell}
\usepackage{enumitem}
\usepackage{subcaption}
\usepackage{float}
\usepackage[table,dvipsnames]{xcolor}
\usepackage{microtype}
\usepackage[capitalize,noabbrev]{cleveref}

\DeclareMathOperator{\Tr}{Tr}

\DeclareMathOperator{\softmax}{softmax}
\DeclareMathOperator{\vecop}{vec}
\DeclareMathOperator{\End}{End}
\DeclareMathOperator{\GL}{GL}
\DeclareMathOperator{\blkdiag}{blkdiag}
\newcommand{\R}{\mathbb{R}}

\newcommand{\norm}[1]{\left\lVert #1 \right\rVert}

\newcommand{\1}{\mathbbm{1}}

\newcommand{\kron}{\otimes}

\newcommand{\Aeff}{A_{\mathrm{eff}}}
\newcommand{\Oeff}{O^{\mathrm{eff}}}
\newcommand{\Leff}{L_{\text{mix}}}
\newcommand{\Lcrw}{L_c^{\mathrm{rw}}}
\newcommand{\TSHA}{T_{\text{SHA}}}
\newcommand{\TMHA}{T_{\text{MHA}}}
\newcommand{\ts}{\text{vec}^r}
\newcommand{\Mh}{M_h}
\newcommand{\Ah}{A^{(h)}}
\newcommand{\LN}{\mathrm{LN}}
\newcommand{\FFN}{\mathrm{FFN}}
\newcommand{\mask}{\mathcal{M}}

\definecolor{bg-gray}{gray}{0.92}

\graphicspath{{content/}}
\newenvironment{xfigure}[1][tbp]{\begin{figure}[#1]}{\end{figure}}
\newcommand{\tabfit}[1]{#1}

\hypersetup{hidelinks} 

\ShortHeadings{From Self-Attention to Connection Laplacian}{Lin, Chen, Li, Wang, Ye and He}
\firstpageno{1}

\begin{document}

\title{From Self-Attention to Connection Laplacian:\\
  A Unified Operator View of Transformers}

\author{%
  \name Binbin Lin \email binbinlin@zju.edu.cn \\
  \addr School of Software Technology, Zhejiang University, China
  \AND
  \name Wei Chen \email weichen.cw@zju.edu.cn \\
  \addr College of Computer Science and Technology, Zhejiang University, China
  \AND
  \name Yalun Li \email yalunli@zju.edu.cn \\
  \addr College of Computer Science and Technology, Zhejiang University, China
  \AND
  \name Wenxiao Wang \email wenxiaowang@zju.edu.cn \\
  \addr School of Software Technology, Zhejiang University, China
  \AND
  \name Jieping Ye \email yejieping.ye@alibaba-inc.com \\
  \addr Alibaba Cloud, China
  \AND
  \name Xiaofei He \email xiaofeihe@cad.zju.edu.cn \\
  \addr College of Computer Science and Technology, Zhejiang University, China%
}

\maketitle

\begin{abstract}
Self-attention is a ubiquitous primitive in modern sequence models, yet its operator-level geometry is only partially understood.
We view a token sequence as a \emph{vector field over the token-position graph} and identify attention as a \emph{connection walk}: messages are aggregated by a nonnegative walk matrix while being \emph{transported} along each edge by a learned linear map.
Within this framework, we prove that \emph{single-head attention (SHA)} is exactly a connection propagation step with \emph{constant} transport, and that \emph{multi-head attention (MHA)} is exactly a \emph{single edge-dependent connection walk} whose effective transport is an attention-gated mixture of headwise transports.
We further clarify the conditions under which the corresponding generator reduces to a \emph{random-walk connection Laplacian}, highlighting the roles of stochasticity, reversibility, and metric-compatible transports.
Empirically, we find that trained Transformers across scales (from 124M to 8B) and structures (encoder/decoder) exhibit geometric structure consistent with our theory: effective attention graphs converge to stable geometric operators in deeper layers, learned transports self-organize into approximate scaled isometries, and both phenomena strengthen consistently with scale.
Overall, the paper provides a precise connection-walk formalism that links self-attention to classical geometric operators, along with a set of operator-level tools for analyzing transformer models from a geometric perspective.
\end{abstract}


\begin{keywords}
  connection Laplacian, self-attention, multi-head attention, transformers, geometric deep learning
\end{keywords}

\section{Introduction}
\label{sec:intro}
Transformers and their self-attention mechanism~\citep{vaswani2017attention} have reshaped modern deep learning by enabling models to capture long-range dependencies across tokens.
In self-attention, each token attends to other tokens through a learned similarity measure, producing a weighted aggregation of value vectors.
The resulting attention matrix can be interpreted as a directed weighted graph where edges encode information flow.
This paper gives an operator-level geometric interpretation of self-attention: token representations form a vector field over the token-position graph (the token graph, for short), and attention acts as a connection walk that mixes tokens while transporting features along edges.
Developing this operator view, we \textit{clarify the operator structure underlying multi-head attention (MHA), the depthwise behavior, and the geometry of token interactions}.

\subsection{Related work}

Existing literature examines attention through various theoretical lenses, which we categorize into graph, tensor, dynamical, energy, kernel, and geometric perspectives.

\textbf{Graph operator view.}
Self-attention can be viewed as a data-dependent linear operator that mixes token features via the attention matrix.
This formulation aligns with graph transformers, where full attention implies all-to-all communication and sparse attention recovers local aggregation~\citep{yun2019gtn}.
Analytic approaches often symmetrize or normalize the attention-induced graph and interpret updates in terms of graph Laplacian smoothing, spectral bias, or oversmoothing~\citep{kreuzer2021rethinking, wu2024maskln}.
These works primarily analyze the scalar mixing graph; in contrast, our operator view also incorporates the value-output maps that transport feature vectors between token fibers.

\textbf{Tensor operator view.}
A parallel line of research treats attention blocks as structured tensor operators, aiming to preserve multi-way correlations or enhance efficiency.
Representative approaches include extending attention to tensor inputs via matricization to couple information across multiple modes~\citep{babiloni2020tesa}, and employing tensor decompositions to reduce parameter complexity~\citep{ma2019tensorized}.
Related work further explores higher-order or multilinear attention variants that alter the algebraic formulation of similarity and aggregation~\citep{tpa2025}.
This line is complementary to ours: tensor methods change or compress the algebra of attention, whereas we identify the connection-walk structure already present in standard MHA.

\textbf{PDE and dynamical-systems view.}
Several works study attention by taking continuous limits, interpreting layerwise updates as discretizations of differential equations.
Continuous-depth Transformer variants make this explicit by treating block parameters as functions of depth and analyzing the resulting non-autonomous neural ODE~\citep{tong2025neuralode}.
Related formulations view the Transformer as a forward-Euler discretization and add trajectory-level regularization to stabilize the evolution~\citep{kan2025ottransformer}.
Our connection-walk formulation supplies an operator-level decomposition whose reversible, directed, and local deformation components naturally align with advection--diffusion--reaction (ADR) effects.

\textbf{Energy-functional view.}
Attention also relates to energy-based updates and associative memory dynamics.
Modern Hopfield network formulations derive attention-like retrieval from an energy function with fixed points corresponding to stored patterns, connecting attention to iterative minimization and memory retrieval~\citep{ramsauer2020hopfield,farooq2025nonlinearhopfield}.
This view clarifies stability and retrieval capacity, and motivates nonlinear variants of attention.
This perspective targets the fixed-point dynamics of retrieval, whereas our operator view characterizes the connection-walk transport enacted by a single attention step.

\textbf{Kernel and approximation view.}
Dot-product attention is also interpreted as a data-dependent kernel operator.
Performer-style methods approximate softmax attention with positive random features to achieve linear-time complexity~\citep{choromanski2020performer}, and surveys systematize a broader family of efficient variants~\citep{tay2023efficient}.
Theoretical analyses characterize attention as learning non-Mercer kernels on Banach spaces, providing representer-style results and universal approximation~\citep{wright2021nonmercer}.
Complementary work also studies learning explicit Transformer kernels~\citep{chowdhury2022transformerkernel}.
These works primarily focus on the scalar affinity kernel or its approximation, whereas our formulation separates walk weights from edgewise feature transports.

\textbf{Geometric and gauge-equivariant view.}
In geometric deep learning, \emph{transport} and \emph{connections} appear explicitly when comparing features across local frames.
Gauge-equivariant networks formalize how local frame changes act on features and constrain valid operators~\citep{cohen2019gauge, bronstein2021gdl}.
On the spectral-geometry side, vector diffusion maps and the graph connection Laplacian provide a principled framework for diffusion on vector fields with transports~\citep{singer2012vdm,bandeira2013cheeger}.
Sheaf neural networks and neural sheaf diffusion equip graphs with vector spaces and linear maps and then define diffusion through sheaf or connection Laplacians~\citep{bodnar2022neuralsheafdiffusion,barbero2022connectionlaplacians}.
These works construct graph neural architectures from prescribed or learned sheaves, while we show that ordinary MHA already induces a connection-valued walk whose effective transport is an attention-gated mixture of headwise value-output maps.

Taken together, these perspectives provide valuable interpretations of attention; yet, to the best of our knowledge, prior work does not explicitly identify standard multi-head attention as a \emph{single edge-dependent connection-walk operator} on a vector field over the token graph.
This requires jointly analyzing the query-key walk weights and the value-output transports, rather than only the attention probabilities or individual heads.
Motivated by this gap, we investigate the precise relationship between self-attention and connection walks and use it to interpret multi-head transport, depthwise behavior, and operator-level diagnostics.
Concretely, we prove that single-head attention (SHA) and MHA are each exactly a connection-walk step (\S\ref{sec:sa-cp}) and identify when the induced generator reduces to a random-walk connection Laplacian (\S\ref{sec:sa-to-cl}); we then turn these operators into diagnostics and report their behavior across trained Transformers from 124M to 8B parameters (\S\ref{sec:experiments}).


\section{Preliminaries: Notation and Attention Operators}
\label{sec:pre}

\textbf{Notation.}
Let $X\in\R^{n\times d}$ collect $n$ token vectors (rows) in $\R^d$, i.e., $X_i\in\R^{1\times d}$ is the feature row at position $i$.
We use row-consistent stacking to switch between matrix form and a single long vector. Define \emph{row-stacking operator} $\ts$ as
\begin{equation}
\ts(X)\;\stackrel{\mathrm{def}}{=}\;\begin{bmatrix}X_1^\top\\ \vdots\\ X_n^\top\end{bmatrix}\in\R^{nd}.
\label{eq:ts-def}
\end{equation}
Equivalently, $\ts(X)=\vecop(X^\top)$ where $\vecop(\cdot)$ denotes the standard column-stacking vectorization.
We use this convention for convenience because right-multiplication of row vectors becomes left-multiplication after stacking.

For $A\in\R^{n\times n}$ (mixing across token positions) and $M\in\R^{d\times d}$ (mixing across channels),
\begin{equation}
\ts(AXM) \;=\; (A\kron M^\top)\,\ts(X).
\label{eq:ts-kronecker}
\end{equation}
Since $\ts(X)=\vecop(X^\top)$ and $(AXM)^\top=M^\top X^\top A^\top$, the identity follows from
$\vecop(UBV)=(V^\top\kron U)\vecop(B)$.

\textbf{Single-head attention (SHA).}
Let $W_Q\in \R^{d\times d_q},W_K\in \R^{d\times d_q},W_V\in \R^{d\times d_v}$ and $W_O\in\R^{d_v\times d}$. Define
$Q=XW_Q, K=XW_K,V=XW_V,$
and the attention (row-wise) probabilities $A=\softmax\!\big((QK^\top)/\tau\big)\in\R^{n\times n},$
where $\softmax$ is applied \emph{row-wise} (optionally after adding a mask), and $\tau>0$ is the temperature.
With row-wise $\softmax$, each row of $A$ sums to one on its support, so $A$ is row-stochastic in the unmasked
(full-support) case and row-stochastic on the admissible set under masking.
The SHA head output is $Y=A\,V\,W_O \;=\; A\,X\,(W_V W_O)$.

\textbf{Definition (Transport Matrix).}
We define the transport matrix as the composite linear map:
\[
M \;\stackrel{\mathrm{def}}{=}\; W_V\,W_O \in \R^{d\times d}.
\]
Attention performs \emph{message passing} on the token graph: token $i$ aggregates messages from tokens $j$ weighted by $A_{ij}$.
We term $M$ the \emph{transport matrix} as it is the per-head channel map transporting each token's features into message space, which is then mixed across positions by $A$.
Writing the output row-wise as $Y_i = \sum_{j=1}^n A_{ij}\, X_j\,M$ reveals that $M$ acts as \textit{a shared linear transport} across all edges within a head.

\begin{remark}
The term ``transport'' highlights the functional role of $M$ as a message transformation and does not imply orthogonality or invertibility. In MHA, different heads yield multiple transported message types mixed in parallel.
\end{remark}

In operator form, SHA is a composition of \emph{position mixing} by $A$, followed by a shared \emph{channel transport} by $M$:
\[
\ts(Y)=\TSHA\,\ts(X),
\;
\TSHA = A\kron M^\top.
\]

\textbf{Multi-head attention (MHA).}
MHA extends this formulation to $H$ parallel heads.
For each head $h \in \{1, \dots, H\}$, let $A^{(h)}$ be the attention matrix induced by query-key pairs $(W_Q^{(h)}, W_K^{(h)})$.
Let the output projection $W_O \in \mathbb{R}^{H d_v \times d}$ be partitioned row-wise into blocks $W_O^{(h)} \in \mathbb{R}^{d_v \times d}$.
The MHA output is the sum of headwise transformations:
\begin{equation}
    Y \;=\; \sum_{h=1}^H A^{(h)} X \underbrace{\big(W_V^{(h)} W_O^{(h)}\big)}_{M_h}
    \label{eq:mha-head-sum}
\end{equation}
where $M_h \in \mathbb{R}^{d \times d}$ is the transport matrix for head $h$.
Applying the row-stacking operator yields the \emph{sum of Kronecker products} structure:
\begin{equation}
    \ts(Y) \;=\; \TMHA \, \ts(X), \ \TMHA \;=\; \sum_{h=1}^H A^{(h)} \otimes M_h^\top.
    \label{eq:mha-sum}
\end{equation}


\section{Connection Walk and Connection Laplacian on Token Graphs}
\label{sec:connlap}

We model token positions as vertices $V=\{1,\dots,n\}$, and admissible interactions by a directed graph $G=(V,E)$.
Full attention corresponds to the complete directed graph, and sparse attention corresponds to a pruned neighborhood.
The matrix entry $A_{ij}$ is read as the nonnegative weight used when updating node $i$ using information from node $j$; row $i$ collects incoming contributions from sources $j$.

A \emph{walk weight} is a nonnegative matrix $A\in\R^{n\times n}$ with support consistent with $E$, i.e., $A_{ij}>0\Rightarrow (j\to i)\in E$.
In random-walk, $A$ is row-stochastic: $\sum_{j=1}^n A_{ij}=1$, so $A_{ij}$ is a probability distribution over sources $j$ for each $i$.

\subsection{Vector field over the position graph.}
A discrete $d$-dimensional vector field on $G$ assigns a feature vector to each vertex: $X\in\R^{n\times d}$, with row vector $X_i\in\R^{1\times d}$ attached to position $i$.
To compare or aggregate vectors located at different vertices, we equip each oriented edge $(j\to i)$ with a linear map
$O_{ij}\in \End(\R^d)\cong\R^{d\times d}$ called an \emph{edge transport} from node $j$ to node $i$.
With our row-vector convention, $O_{ij}$ acts by right multiplication: $X_j\mapsto X_jO_{ij}$.
When such maps are invertible, orthogonal, or approximately orthogonal, the operator approaches the usual setting of a metric connection; however, invertibility is not required for exact attention identification below.

\subsection{Connection propagation operator}
Given $(A,O)$ and a vector field $X$, the connection propagation operator $\mathcal{T}(A,O)$ defines one step of message passing:
\begin{equation}
(\mathcal{T}(A,O)X)_i \;\stackrel{\mathrm{def}}{=}\; \sum_{j=1}^n A_{ij}\, X_j\, O_{ij}.
\label{eq:conn-prop}
\end{equation}
When $O_{ij}\equiv I_d$, this reduces to ordinary random-walk averaging: $(\mathcal{T}(A,I)X)_i=\sum_j A_{ij}X_j$.

\subsection{Random-walk connection Laplacian}
The random-walk connection Laplacian measures discrepancy between $X_i$ and transported average from neighbors:
\begin{equation}
\begin{aligned}
(\mathcal{L}_{c}^{\mathrm{rw}}(A,O)X)_i
&\;\stackrel{\mathrm{def}}{=}\;
X_i - \sum_{j=1}^n A_{ij}\,X_j\,O_{ij} \\
&\;=\; (I-\mathcal{T}(A,O))X .
\end{aligned}
\label{eq:conn-lap-rw}
\end{equation}
If $A$ is row-stochastic, then any parallel section satisfying $X_i=X_jO_{ij}$ for all $j$ with $A_{ij}>0$ lies in $\ker(\mathcal{L}_{c}^{\mathrm{rw}}(A,O))$,
\[
\sum_j A_{ij}X_jO_{ij}=\sum_j A_{ij}X_i = X_i\sum_j A_{ij}=X_i.
\]

\subsection{Block operator form (token stacking)}
Define block matrix $T(A,O)\in\R^{nd\times nd}$ with $d\times d$ blocks
\begin{equation}
T(A,O)_{ij} \;\stackrel{\mathrm{def}}{=}\; A_{ij}\,O_{ij}^{\top}.
\label{eq:block-T}
\end{equation}
Due to the row-stacking convention $\ts(X) = \vecop(X^\top)$, the transport maps appear as transposes in the blocks:
\[
(\mathcal{T}(A,O)X)_i^\top=\sum_j A_{ij}O_{ij}^\top X_j^\top .
\]
Consequently,
\[
\ts(\mathcal{T}(A,O)X) \;=\; T(A,O)\,\ts(X) .
\]


\section{Self-Attention as Connection Propagation}
\label{sec:sa-cp}

Let $G$ be the directed token graph on $V=\{1,\dots,n\}$.
We use the term \emph{connection walk} for connection propagation with edge transports $O_{ij}\in\End(\R^d)$.
Note that we do not assume $O_{ij}\in\GL(d)$; standard attention transports can be low-rank or singular because the value-output pathway may factor through a lower-dimensional head space.

\subsection{SHA as connection walk propagation}

\begin{theorem}\label{thm:sha}
Assume $A$ is row-stochastic and the edge transport is constant, $O_{ij}\equiv M$ for all active edges $(j\to i)$.
Then the single-head attention (SHA) $Y=AXM$ is exactly the connection propagation step
\[
  Y \;=\; \mathcal{T}(A,O)X .
\]
Equivalently, in stacked form, $\TSHA=T(A,O)=A\otimes M^\top$.
\end{theorem}

\begin{proof}
For every target token $i$,
\[
  Y_i \;=\; \sum_j A_{ij}X_jM \;=\; \sum_j A_{ij}X_jO_{ij},
\]
which is Eq.~\eqref{eq:conn-prop} with $O_{ij}\equiv M$.
The Kronecker representation follows from Eq.~\eqref{eq:block-T}.
\end{proof}

\subsection{MHA as edge-dependent connection walk}
\label{sec:mha}

We next show that multi-head attention (MHA) is exactly a single connection propagation step with an edge-dependent transport.
Recall from Eq.~\eqref{eq:mha-sum} that
\[
\ts(Y) \;=\; \TMHA\,\ts(X),
\;
\TMHA=\sum_{h=1}^H \big(\Ah\kron \Mh^\top\big).
\]
Define the \emph{effective walk weights} (mean attention) by
\begin{equation}
\Aeff(i,j) \;\stackrel{\mathrm{def}}{=}\; \frac{1}{H}\sum_{h=1}^H \Ah_{ij}\;\ge 0.
\label{eq:Aeff-def}
\end{equation}
When $\Aeff(i,j)>0$, define the \emph{effective edge transport} by the $\Aeff$-weighted average
\begin{equation}
\Oeff_{ij}
\;\stackrel{\mathrm{def}}{=}\;
\frac{ \frac{1}{H} \sum_{h=1}^H \Ah_{ij}\,\Mh}{\Aeff(i,j)}
\;\in\R^{d\times d}.
\label{eq:Oeff-def}
\end{equation}
If $\Aeff(i,j)=0$, $\Oeff_{ij}$ can be chosen arbitrarily, for instance as $I_d$, because the edge carries zero weight.

\begin{theorem}[Exact reduction to a scaled edge-dependent connection step]
\label{thm:mha}
Let $\TMHA$ be given by Eq.~\eqref{eq:mha-sum}. Let $\Aeff$ and $\Oeff$ be defined by
Eqs.~\eqref{eq:Aeff-def}--\eqref{eq:Oeff-def}. Then, with $T(\cdot,\cdot)$ defined in Eq.~\eqref{eq:block-T},
\begin{equation}
\TMHA \;=\; H \cdot T(\Aeff,\Oeff).
\label{eq:TMHA-connection}
\end{equation}
Equivalently, the MHA output satisfies
\begin{equation}
\begin{aligned}
Y_i \;&=\; H \sum_{j=1}^n \Aeff(i,j)\,X_j\,\Oeff_{ij} \\
\;&=\; H\,(\mathcal{T}(\Aeff,\Oeff)X)_i .
\end{aligned}
\label{eq:mha-connection-row}
\end{equation}
\end{theorem}

\begin{proof}
We verify the equality block-wise. By Eq.~\eqref{eq:mha-sum}, the $(i,j)$ block of $\TMHA$ equals $\sum_{h=1}^H \Ah_{ij}\Mh^\top$.

\textbf{Case 1: active edge ($\Aeff(i,j)>0$).}
Taking the transpose of Eq.~\eqref{eq:Oeff-def},
\[
{\Oeff_{ij}}^\top
= \frac{\frac{1}{H}\sum_h \Ah_{ij}\Mh^\top}{\Aeff(i,j)} .
\]
Multiplying by $H\Aeff(i,j)$ yields
\[
H\Aeff(i,j){\Oeff_{ij}}^\top
= \sum_{h=1}^H \Ah_{ij}\Mh^\top,
\]
which matches the MHA block.

\textbf{Case 2: inactive edge ($\Aeff(i,j)=0$).}
Since $\Ah_{ij}\ge 0$, the zero average implies $\Ah_{ij}=0$ for all heads $h$.
The MHA block is therefore zero, and the connection block $H\Aeff(i,j){\Oeff_{ij}}^\top$ is also zero.

Thus, every block matches, proving Eq.~\eqref{eq:TMHA-connection}.
\end{proof}

\begin{remark}[Scale convention]
The factor $H$ is a bookkeeping consequence of defining $\Aeff$ as the average attention over heads.
Equivalently, one may absorb this scalar into the effective transports.
Our geometric diagnostics focus on normalized walk weights and relative transport geometry.
\end{remark}


\section{From Connection Walks to Connection Laplacians}
\label{sec:sa-to-cl}

Theorem~\ref{thm:mha} identifies MHA exactly as a scaled connection walk with edge-dependent transports.
We now discuss several spectral and geometric properties of induced operator, and relate them to classical metric connection Laplacian.

\subsection{Effective attention weights}

\textbf{Row-stochasticity.}
Each head weight matrix $\Ah$ is produced by a row-wise softmax, with masking restricting support, hence $\sum_j \Ah_{ij}=1$ on its admissible set.
Averaging preserves row sums, so $\sum_j\Aeff(i,j)=1$.
Consequently, $\Aeff \1=\1$ and the spectral radius of $\Aeff$ is $1$ for the finite-state row-stochastic walk.
Increasing $\|\Aeff\|_F$ should therefore be interpreted as increasing the concentration of the walk.
Singular-value amplification can occur for non-normal or column-concentrated directed walks.

\textbf{Asymmetry and reversibility.}
In general, $\Aeff$ is not symmetric: causal masking makes the graph directed by construction, and even in bidirectional encoders, the learned weights need not satisfy $\Aeff(i,j)=\Aeff(j,i)$.
A useful structural notion is reversibility: if $\pi$ is a stationary distribution of $\Aeff$, then $\Aeff$ is reversible iff
\begin{equation}
  \pi_i \Aeff(i,j) \;=\; \pi_j \Aeff(j,i)\quad \forall\, i,j .
  \label{eq:reversible}
\end{equation}
Reversibility is the discrete condition that restores self-adjoint spectral theory.
Without it, the walk is a \textit{directed or non-reversible diffusion}, and the generator is generally non-normal, consistent with directed Laplacian frameworks~\citep{chung2005directed,veerman2020primer}.

\subsection{Effective transport}

\textbf{Orthogonality.}
The classical metric graph connection Laplacian assumes transports in $O(d)$ or $SO(d)$ so that $O_{ij}$ preserves inner products and lengths~\citep{singer2012vdm,bandeira2013cheeger}.
In our identification, the analogous sufficient condition is ${\Oeff_{ij}}^\top\Oeff_{ij}\approx \mu_{ij} I_d$ on active edges, with $\mu_{ij}>0$ allowing a scaled-isometric diagnostic.
Since $\Oeff_{ij}$ is a linear combination of the head transports $\Mh$, it is not necessarily invertible or orthogonal in general.
This makes the empirical emergence of approximate scaled-isometry in Section~\ref{sec:experiments} nontrivial.

\textbf{Metric deformation and separability.}
When $\Oeff_{ij}$ is not orthogonal, propagation can alter norms and angles, inducing metric deformation.
If $\Oeff_{ij}$ is nonsingular, the polar decomposition gives
\[
  \Oeff_{ij} \;=\; R_{ij}\,S_{ij}, \quad R_{ij}\in O(d),\quad S_{ij}\succeq 0 .
\]
$R_{ij}$ captures \emph{length-preserving transport} (pure connection diffusion), while $S_{ij}$ captures
\emph{anisotropic scaling and shearing} (metric deformation).
For singular transports, the same decomposition uses a partial isometry.
In practice, the extent to which deformation is realized inside attention is an empirical question measured by the diagnostics in Section~\ref{sec:experiments}.

\subsection{When does the attention connection walk become a connection Laplacian?}
\label{subsec:when-laplacian}

Let the normalized MHA propagator in block form be $\bar T_{\mathrm{MHA}}=\TMHA/H=T(\Aeff,\Oeff)$, and define generator
\begin{equation}
  \Leff \;=\; I - \bar T_{\mathrm{MHA}},
\end{equation}
which is a well-defined generator in the form $I-\text{(one-step propagator)}$.
However, without reversibility and metric-compatible transports, $\Leff$ is generally not self-adjoint and $\langle X,\Leff X\rangle_{\pi}$ need not be nonnegative.
It's thus better understood as a directed connection-walk generator rather than a classical metric connection Laplacian.
Under a weighted inner product $\langle X, Y\rangle_{\pi}=\sum_i\pi_i\langle X_i, Y_i\rangle$, one may decompose the generator into symmetric and skew components: the symmetric part governs dissipation, while the skew part captures directed drift, advection-like behavior, and transient growth typical of non-reversible operators.

A central benefit of a metric connection Laplacian is an associated nonnegative Dirichlet form.
If (i) the walk is reversible as in Eq.~\eqref{eq:reversible} and (ii) transports are metric-compatible with inverse consistency, e.g., $\Oeff_{ji}=(\Oeff_{ij})^{-1}$ on bidirectional edges, then $\mathcal{T}(\Aeff,\Oeff)$ is self-adjoint in the corresponding geometry and the generator becomes PSD.
In that regime,
\begin{equation}
  \mathcal{E}_{\pi}(X)
  \;\stackrel{\mathrm{def}}{=}\;
  \tfrac12\sum_{i,j}\pi_i \Aeff(i,j)\,\|X_i - X_j\Oeff_{ij}\|_2^2
  \label{eq:energy-pi}
\end{equation}
is the standard connection Dirichlet energy and vanishes exactly on parallel vector fields~\citep{lin2013parallel}.
For generic directed and non-isometric attention, we use analogous energy quantities only as diagnostics, not as PSD quadratic forms of a self-adjoint Laplacian.


\section{Operator-Level Interpretation, Dynamics and Diagnostics}
\label{sec:insights}

Identifying (multi-head) self-attention as a \emph{connection walk} on a token graph transforms it into an explicit operator family on fields, enabling geometric \emph{design levers} (parameterizations, constraints, regularizers) and \emph{diagnostics} (energy, curvature and holonomy, metric-compatibility) that are hard to motivate from ``similarity + aggregation'' alone.

\subsection{Attention-level vs. block-level interpretation}
\label{subsec:block-scope}

The exact connection-walk identification is an attention-sublayer statement.
A full Transformer block additionally contains residual connections, normalization layers, and pointwise nonlinear FFNs.
For a pre-norm block, a common schematic form is
\[
  U = X + \mathrm{Attn}(\LN(X)),
  \;
  Y = U + \FFN(\LN(U)).
\]
These components are essential for optimization and representation, but have a different operator order on the token graph.
LayerNorm and FFN act independently at each token within a layer, so their Jacobians are block diagonal across token positions.
They do not introduce same-layer cross-token edges as attention.
Nevertheless, they reshape local fiber coordinates and therefore influence the Q/K/V maps and the realized connection geometry in subsequent layers.

For example, for a row vector $x_i$, LayerNorm has the form
\[
\LN(x_i)=\gamma\odot
\frac{x_i-\mu_i\1^\top}{\sqrt{\sigma_i^2+\epsilon}}+\beta,
\;
\mu_i=\tfrac1d x_i\1,
\]
with variance computed across channels.
Thus its Jacobian over a sequence is $\blkdiag(J^{\mathrm{LN}}_1,\dots,J^{\mathrm{LN}}_n)$.
Similarly, for an FFN $f(x)=\phi(xW_1+b_1)W_2+b_2$ in row-vector notation, the local linearization around a layer state has Jacobian
\[
  \begin{aligned}
  \mathcal{J}_{f}(X)&=\blkdiag(J_1,\dots,J_n),\\
  J_i&=W_1D_{\phi'(x_iW_1+b_1)}W_2 .
  \end{aligned}
\]
This is the discrete analog of a zero-order reaction term: it transforms the fiber at each node without adding new edges.

\subsection{Design and diagnostics}
\label{subsec:design-diagnostics}

A connection-walk attention layer specifies two coupled objects: (i) a walk $\Aeff$ on the token graph, and (ii) an edge transport $\Oeff_{ij}$ that maps features between fibers.
MHA provides an efficient parameterization of edge-dependent transports: heads form a low-dimensional \emph{transport dictionary}, while attention performs \emph{edgewise dictionary selection}.
Increasing the number of heads enlarges the admissible family of edge transports without learning a full per-edge matrix.

The operator view also reveals that MHA is an implementation of a connection walk, and the design problem becomes: \emph{how should one parameterize a stable yet expressive family of mixing operators and edge transports?}

This opens a broader implementation space. For example, spectral and functional-calculus constructions, kernelized parameterizations that control locality and approximate global mixing, and structured or constraint-preserving transports (near-isometries, inverse consistency) that bias the operator toward a metric connection Laplacian regime. More generally, the ``dictionary'' perspective extends beyond heads: one may mix over any low-dimensional operator basis (polynomial/spectral modes, wavelet-like primitives, learned linear dynamical modes), trading per-edge flexibility for parameter efficiency, stability, and interpretability.

For diagnostics, the operator view suggests layerwise or headwise measurements: metric-compatibility scores such as $\|{\Oeff_{ij}}^\top\Oeff_{ij}-\mu_{ij}I\|$; reversibility and symmetrizability tests for $\Aeff$; inverse-consistency tests for $\Oeff$; and curvature or holonomy proxies via loop-consistency of transports along short cycles.
Such quantities may help identify outlier layers or heads and provide principled signals for future pruning, normalization, or regularization studies.

\subsection{Extension to higher-order tensor data: connection walks beyond vector fields}
\label{subsec:tensor-extension}

The same connection-walk formalism extends from node vectors $x_i\in\R^d$ to matrix- and tensor-valued fibers by specifying the fiber space and transports.
Concretely, replace $\R^d$ by a representation space $\mathcal{F}$, such as $\R^{d_1\times d_2}$, and let transports act through appropriate representations, such as Kronecker-structured actions.
In vision and video, the base graph can be a grid or a space-time graph; $\Aeff$ encodes the domain topology, while transports align local frames across space-time.
Curvature then measures alignment inconsistency around loops, providing a principled notion of non-flat representation geometry for tensor-structured modalities.

\subsection{ADR view of Transformer blocks}
\label{subsec:adr-objectives}

A connection walk layer mixes neighbors and transports features across edges.
Its action naturally decomposes into three effects.
\textbf{Diffusion} arises from the \emph{symmetric or reversible} component of the walk, smoothing and spreading information across tokens. 
\textbf{Advection} arises from the \emph{antisymmetric, drift-like} component, producing directed, non-reversible flow. 
\textbf{Reaction} captures \emph{local reshaping}, including non-isometric parts of $\Oeff_{ij}$ and pointwise nonlinearities (e.g., FFNs), which amplify, contract, or rotate features in place. 
Thus, a connection walk can be read as an ADR operator on a vector field over the token graph.



The most faithful continuous analog of attention is generally nonlocal: for a continuous token-position variable $s$ and feature field $u(t,s)$, a connection-walk propagator is
\[
  (\mathcal{T}_t u)(s)=\int a_t(s,s')\,O_t(s,s')\,u(t,s')\,ds' .
\]
When the kernel localizes, the generator admits a local ADR approximation
\[
  F_t(v)=\nabla^{\!*}(D_t\nabla v)-b_t\cdot\nabla v+C_t v+\Phi_t(v),
\]
where $D_t$ is a diffusion tensor, $b_t$ is drift, $C_t$ captures linear zero-order effects such as non-isometric transport, and $\Phi_t$ captures nonlinear reaction.
A pre-LN block is naturally modeled as a normalized-forcing evolution
\[
  \partial_t u = F_t(\LN(u)),
\]
where the residual state evolves in ambient feature space while the vector field is evaluated on normalized inputs.
A post-LN block is better viewed as a projected evolution
\[
  \partial_t u = J_{\mathrm{LN}}(u)F_t(u),
\]
where normalization projects the infinitesimal ADR update onto the tokenwise normalization manifold.
These continuous formulations are interpretive; the theorem-level exactness remains the attention-sublayer connection-walk result.

The ADR view suggests measurable geometric quantities, such as connection-energy residuals and loop-consistency penalties, that may monitor depthwise stabilization or inspire future low-curvature objectives.
We treat such objectives as future directions rather than validated replacements.


\section{Experiments}
\label{sec:experiments}

In this section, we empirically investigate the geometric properties of multi-head attention in trained Transformers. We combine quantitative analysis across model families (from 124M to 8B parameters) with direct operator visualization. Specifically, we examine the following properties:
\begin{enumerate}[leftmargin=1em, topsep=0pt, itemsep=2pt, parsep=0pt]
    \item \textbf{Geometric Stability:} Whether the graph topology induced by the learned effective walk weights $\Aeff$ (row-stochastic by nature) stabilizes across layers and whether the distribution sharpens (more concentrated).
    \item \textbf{Transport Rigidity:} Whether the learned effective transports $\Oeff$ approximate scaled orthogonal operators, suggesting geometry-preserving transformations.
    \item \textbf{Scaling Robustness:} How these geometric properties vary with model scale and context length.
\end{enumerate}
We hypothesize that the geometric properties may vary across layer regimes. Since the initial and final layers perform embedding transformations and task-specific projections, we focus our analysis on the \textit{middle layers}, where the evolution of internal representations occurs primarily.

\subsection{Experimental setup}

\textbf{Models and data.}
The main study evaluates decoder-only Transformers across several model families: GPT-2 (Small, Medium, Large, XL), Qwen2.5 (3B, 7B), Qwen3 (4B, 8B), Llama-3.2 (3B), and Llama-3 (8B).
Experiments are conducted on the WikiText-2 dataset~\footnote{https://huggingface.co/datasets/Salesforce/wikitext.}.
We randomly sample 1,024 sequences, truncate them to a context length of 64 tokens unless otherwise specified, and report average statistics.
Appendix~\ref{app:bert_encoder} additionally reports encoder-only BERT diagnostics, and Appendix~\ref{app:connection_energy} reports connection-energy measurements.

\begin{xfigure}[t]
    \centering
    \begin{subfigure}[b]{0.48\linewidth}
        \centering
        \includegraphics[width=\linewidth]{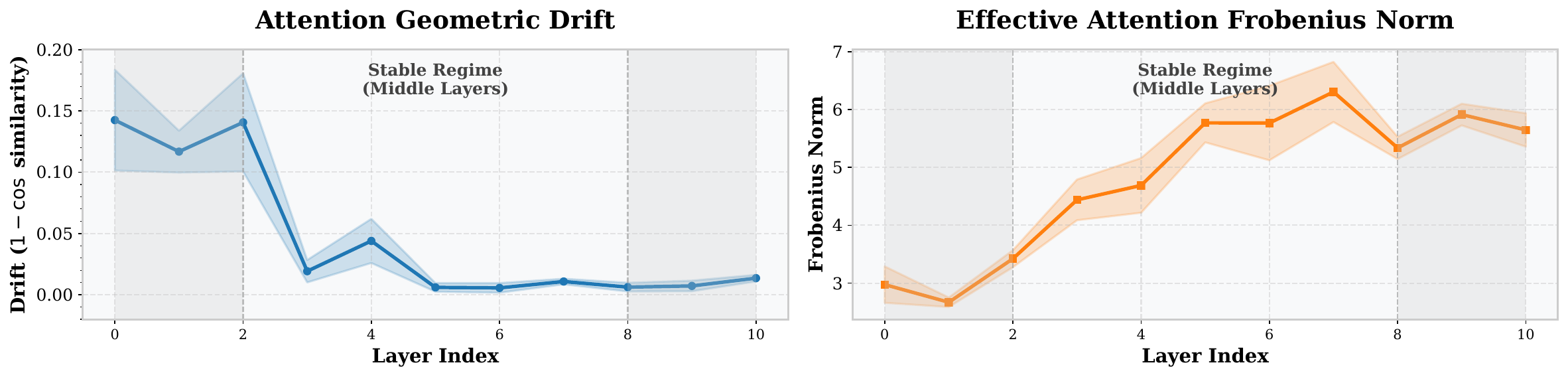}
        \caption{GPT-2 Small: Geometric Drift}
        \label{fig:drift_gpt2}
    \end{subfigure}
    \hfill
    \begin{subfigure}[b]{0.48\linewidth}
        \centering
        \includegraphics[width=\linewidth]{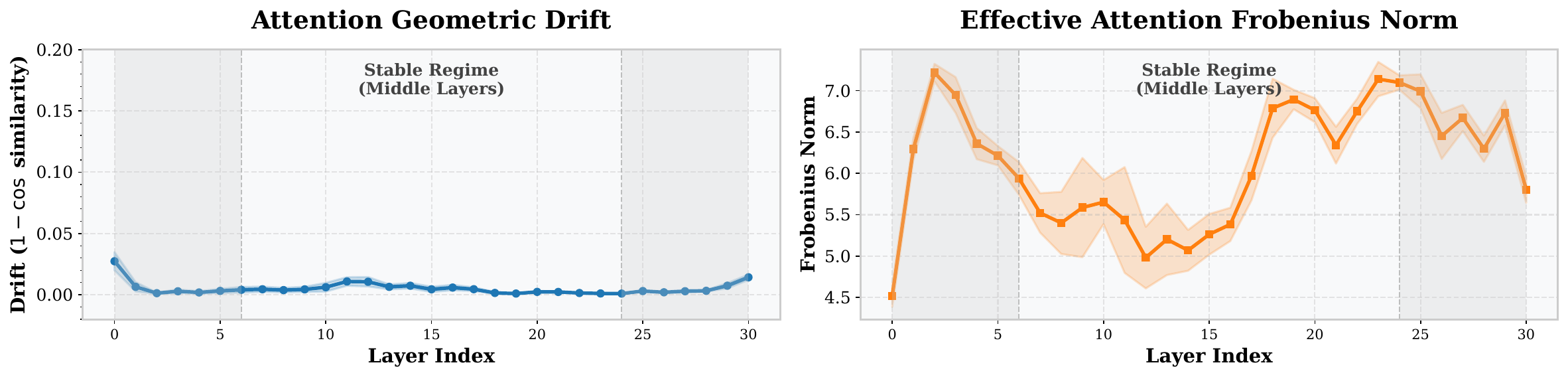}
        \caption{Llama-3-8B: Geometric Drift}
        \label{fig:drift_llama3}
    \end{subfigure}
    
    \vspace{0.1cm} 

    \begin{subfigure}[b]{0.48\linewidth}
        \centering
        \includegraphics[width=\linewidth]{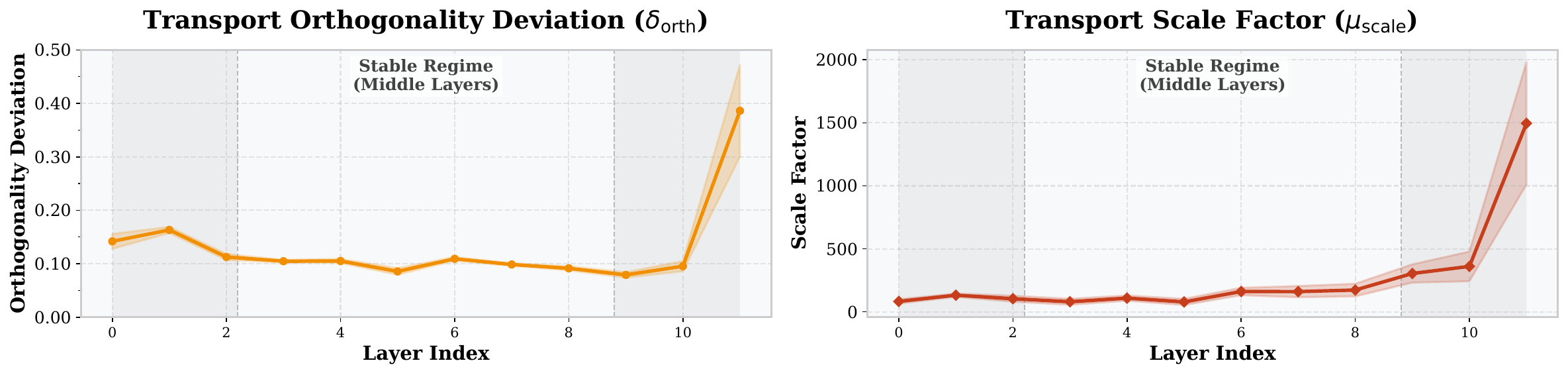}
        \caption{GPT-2 Small: Orthogonality Deviation}
        \label{fig:ortho_gpt2}
    \end{subfigure}
    \hfill
    \begin{subfigure}[b]{0.48\linewidth}
        \centering
        \includegraphics[width=\linewidth]{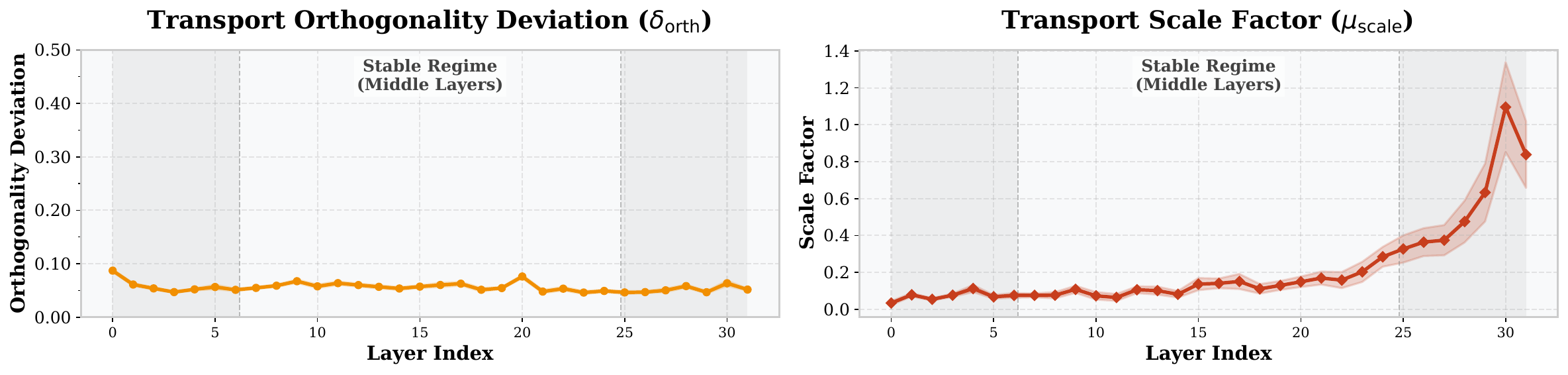}
        \caption{Llama-3-8B: Orthogonality Deviation}
        \label{fig:ortho_llama3}
    \end{subfigure}

    \caption{\textbf{Layerwise diagnostics of effective walks and transports.}
    Shaded regions indicate the initial and final layers, which are often dominated by embedding and output-projection effects; the discussion focuses on the \textit{stable middle-layer regime}.
    \textbf{(a, b) Geometric Drift:} Both models exhibit stabilization patterns. GPT-2 Small shows higher variance, whereas Llama-3-8B reaches a lower-drift middle-layer regime ($d_A^{\ell \rightarrow \ell+1} \approx 0$).
    \textbf{(c, d) Transport Orthogonality:} Llama-3-8B has lower orthogonality deviation under the scaled-isometry diagnostic ($\delta_{\text{orth}} \approx 0.05$ in middle layers), whereas GPT-2 Small exhibits larger deviation.}
    \label{fig:scaling_comparison}
\end{xfigure}

\textbf{Metrics.}
We track the layerwise evolution of the effective operators defined in Eqs.~\eqref{eq:Aeff-def}--\eqref{eq:Oeff-def}.

\begin{itemize}[leftmargin=1em, topsep=0pt, itemsep=2pt, parsep=0pt]
    \item \textbf{Geometric drift ($d_A$).}
    To measure stability of the token graph, we compute the cosine distance between effective attention matrices of layers over a support mask $\mask$:
    \begin{equation*}
        d_{A}^{\ell \to \ell+1} = 1 - \mathrm{cos\_sim}(\mathrm{vec}_{\mask}(\Aeff^{\ell+1}), \mathrm{vec}_{\mask}(\Aeff^{\ell})).
    \end{equation*}
    For decoder-only models, $\mask$ is causal lower-triangular support, so this reduces to flattening $\mathrm{tril}(\Aeff)$.
    For bidirectional encoders like BERT, $\mask$ is the full valid non-special-token attention mask.
    We also monitor $\|\Aeff^\ell\|_F$; for row-stochastic $\Aeff$, a larger Frobenius norm indicates a more localized walk distribution.
    
    \item \textbf{Transport orthogonality ($\delta_{\text{orth}}$).}
    To quantify closeness to a scaled orthogonal matrix, define the local Gram matrix $G_{ij}^\ell = {O_{ij}^{\mathrm{eff},\ell}}^\top O_{ij}^{\mathrm{eff},\ell}$ and mean scale $\mu_{\text{scale}}^\ell(i,j) = \frac{1}{d}\Tr(G_{ij}^\ell)$.
    We report the relative deviation
    \begin{equation}
        \delta_{\text{orth}}^\ell(i,j) = \frac{\|G_{ij}^\ell - \mu_{\text{scale}}^\ell(i,j) I_d\|_F}{\mu_{\text{scale}}^\ell(i,j)\, d}.
        \label{eq:transport_orth}
    \end{equation}
    The normalization yields a per-entry RMS deviation, making the metric comparable across hidden dimensions.
    We report weighted means across valid token pairs.
\end{itemize}

\subsection{Evolution of token geometry and transport}

We examine the layerwise evolution of geometric properties, comparing a \textit{small early} model (GPT-2 Small) and a \textit{large modern} model (Llama-3-8B) to assess the universality of our findings. Full results are presented in Appendix~\ref{app:full_plots}.

\textbf{Geometric stability.}
As shown in Figure~\ref{fig:scaling_comparison} (Top), the geometric drift $d_A$ exhibits a stabilization pattern across both models, with a more pronounced low-drift middle-layer regime in Llama-3-8B.
GPT-2 Small shows higher variance, whereas Llama-3-8B rapidly enters a regime where adjacent-layer effective walks are highly similar.
This suggests that the topology of token interactions can become approximately stationary across depth.
Additionally, we observe that $\norm{\Aeff^\ell}_F$ consistently increases with depth. This implies that while the geometry stabilizes, the diffusion process becomes progressively more concentrated, preserving distinct features rather than oversmoothing them.

\textbf{Transport rigidity.}
Figure~\ref{fig:scaling_comparison} (Bottom) shows that $\Oeff$ approaches an approximate scaled-isometric regime under the diagnostic in Eq.~\eqref{eq:transport_orth}.
In Llama-3-8B, the middle-layer orthogonality deviation can be as low as $\delta_{\text{orth}}\approx 0.05$.
GPT-2 Small exhibits larger deviations, whereas the larger modern model in this comparison has lower scaled-isometry deviation under the same diagnostic.
Note that this observation is descriptive: it does not establish that lower $\delta_{\text{orth}}$ causes better downstream performance.

\begin{xfigure}[t]
    \centering
    \begin{subfigure}[b]{\linewidth}
        \centering
        \includegraphics[width=\linewidth]{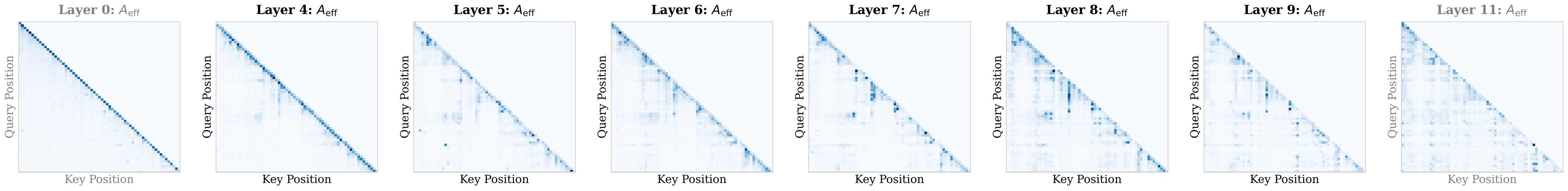}
        \caption{Evolution of Effective Walk Weights $\Aeff$ (First token attention sink removed)}
        \label{fig:heatmap_aeff}
    \end{subfigure}
    
    \vspace{0.1cm}

    \begin{subfigure}[b]{\linewidth}
        \centering
        \includegraphics[width=\linewidth]{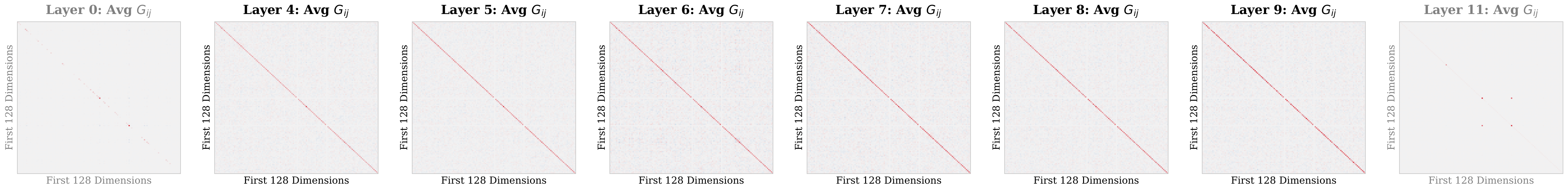}
        \caption{Evolution of Transport Gram Matrix $G_{ij}$ (Subset of 128 channels)}
        \label{fig:heatmap_gram}
    \end{subfigure}

    \caption{\textbf{Visualization of Learned Operators (GPT-2 Small).} 
    \textbf{(a) Effective walk weights ($\Aeff$):} With the attention sink removed, middle layers exhibit a coherent structure, consistent with a stable local-to-global topology.
    \textbf{(b) Transport Gram structure ($G_{ij}$):} The Gram matrix is strongly diagonal-dominant. The suppression of off-diagonal elements is consistent with approximate scaled-isometry under the proposed diagnostic, while not ruling out non-isometric deformation in general.}
\label{fig:heatmaps}
\end{xfigure}

\subsection{Visualizing the effective operators}

To provide intuition beyond scalar metrics, we directly visualize the learned operators $\Aeff$ and the averaged Gram matrix of the transports $G_{ij}$ for GPT-2 Small in Figure~\ref{fig:heatmaps}.

\textbf{Attention as walk weights ($\Aeff$).}
Figure~\ref{fig:heatmap_aeff} visualizes the effective walk weights after removing the first-token attention sink.
Middle layers exhibit a coherent structure, consistent with a stable local-to-global topology rather than random sparse mixing.

\textbf{Transport Gram structure ($G_{ij}$).}
Figure~\ref{fig:heatmap_gram} visualizes the observed diagonal dominance of the mean transport Gram matrix.
Quantitatively, diagonal elements are approximately $15$--$20\times$ larger than off-diagonal elements across many middle-layer regimes (Appendix~\ref{app:orth_stats}).
This is consistent with approximate scaled-isometry under our diagnostic, while still allowing non-isometric deformation in general.

\subsection{Analysis across scales and context lengths}
\label{subsec:analysis-scales-context-lengths}

\begin{table}[t]
\centering
\caption{\textbf{Quantitative analysis of Transport Orthogonality ($\delta_{\text{orth}}$) across model families.}
We report deviations from scaled isometry across the initial, middle, and final layers. Lower values indicate closer agreement with this diagnostic, not necessarily better downstream task performance.
}
\label{tab:orth_stats}
\tabfit{%
    \begin{small}
    \begin{tabular}{lccc} 
    \toprule

    & \multicolumn{3}{c}{Relative Layer Depth} \\
    \cmidrule(lr){2-4}
    
    Model & \makecell{Initial \\ (20\%)} & \makecell{Middle \\ (60\%)} & \makecell{Final \\ (20\%)} \\
    \midrule
    
    \rowcolor{bg-gray} 
    \multicolumn{4}{l}{\textit{GPT-2 Family}} \\ 
    \hspace{1em} Small (124M)   & 0.152 & 0.096 & 0.284 \\
    \hspace{1em} Medium (355M)  & 0.180 & 0.089 & 0.228 \\
    \hspace{1em} Large (774M)   & 0.102 & 0.066 & 0.206 \\
    \hspace{1em} XL (1.5B)      & 0.089 & 0.065 & 0.112 \\
    
    \addlinespace[0.3em] 

    \rowcolor{bg-gray}
    \multicolumn{4}{l}{\textit{Qwen2.5 Family}} \\
    \hspace{1em} Base 3B      & 0.083 & 0.084 & 0.114 \\
    \hspace{1em} Base 7B      & 0.066 & 0.067 & 0.094 \\

    \addlinespace[0.3em] 
    
    \rowcolor{bg-gray}
    \multicolumn{4}{l}{\textit{Qwen3 Family}} \\
    \hspace{1em} Base 4B      & 0.079 & 0.057 & 0.053 \\
    \hspace{1em} Base 8B      & 0.064 & \textbf{0.052} & \textbf{0.051} \\
    
    \addlinespace[0.3em]
    
    \rowcolor{bg-gray}
    \multicolumn{4}{l}{\textit{Llama-3 Family}} \\
    \hspace{1em} Base 8B      & \textbf{0.062} & 0.057 & 0.052 \\
    
    \bottomrule
    \end{tabular}
    \end{small}
}
\end{table}

Finally, we systematically analyze how these geometric properties correlate with model capacity and context length.

\textbf{Scaling with Model Size.}
Table~\ref{tab:orth_stats} reports the mean orthogonality deviation across different model families. We observe a robust scaling trend: \textit{larger models tend to learn more rigid geometric transports.}
Deviations generally decrease as model size increases, with larger, modern models (Qwen3, Llama-3) achieving lower deviations ($\approx 0.05$). This trend suggests that the emergence of quasi-isometric transport is not an accidental artifact of small models, but rather a structural preference that stronger models exhibit more distinctly, while we do not present this as a general downstream performance predictor.

\begin{figure}[t]
    \centering
    \includegraphics[width=0.7\linewidth]{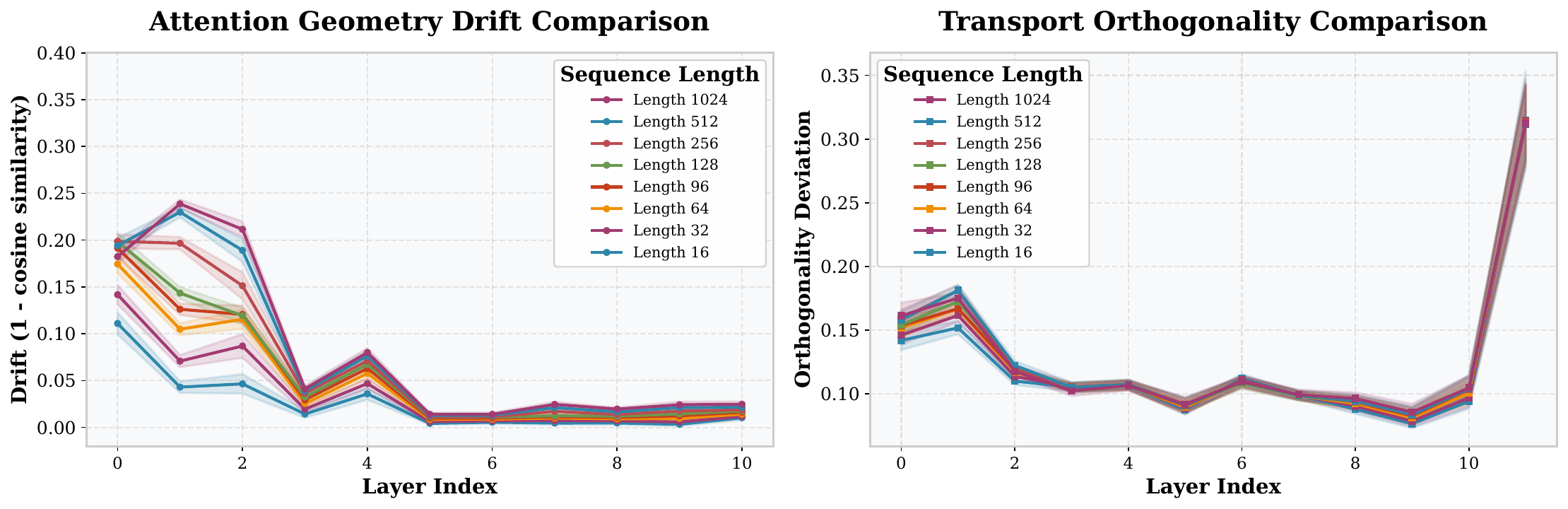}
    \caption{\textbf{Sensitivity to context length (GPT-2 Small).}
    We vary the input context length from $16$ to $1{,}024$ tokens. Longer contexts increase early-layer drift and mildly affect orthogonality, while deeper layers recover a similar low-drift and low-deviation profile in this tested range.}
\label{fig:ablation_main}
\end{figure}

\textbf{Sensitivity to context length.}
Figure~\ref{fig:ablation_main} evaluates GPT-2 Small as context length varies from $16$ to $1{,}024$ tokens.
Longer contexts increase early-layer drift and mildly weaken orthogonality, which is expected because maintaining a coherent geometry over a larger graph is more difficult.
However, the overall profile remains similar: deeper layers recover a low-drift and low-deviation regime within the tested range.
Appendix~\ref{app:ablation} reports full-length ablations across model families, and Appendix~\ref{app:long_context} provides a supporting extended-range diagnostic.

\textbf{Encoder-only and energy diagnostics.}
Appendix~\ref{app:bert_encoder} reports BERT experiments using the full valid bidirectional non-special-token mask, and Appendix~\ref{app:connection_energy} reports connection-energy diagnostics for GPT-2 Small and BERT.
These results suggest that the diagnostic phenomena are not limited to decoder-only models, although layerwise patterns remain architecture-dependent and descriptive.

In summary, the experiments show diagnostic signatures captured by the connection-walk formalism: \textit{effective walks often stabilize across depth, and effective transports often approach an approximate scaled-isometric regime}.


\section{Limitations}
\label{sec:limitations}

First, the exact operator identification is an attention-sublayer result.
A full Transformer block also contains residual connections, LayerNorm or RMSNorm, and pointwise nonlinear FFNs.
We interpret these as local or zero-order components in an ADR view, but we do not provide a complete theorem for the full nonlinear block composition.

Second, the empirical study is descriptive and diagnostic.
We do not establish that lower drift or more orthogonal transports causally improve downstream performance, nor do we show that enforcing these properties yields better models.

Third, the main empirical evaluation focuses on pretrained models under controlled text probes.
Additional encoder-only and longer-context diagnostics broaden the scope, but a systematic study of encoder-decoder models, multimodal settings, and task-level behavior remains future work.

Fourth, computing effective edge transports can be expensive, and our diagnostics rely on sampled sequences and practical normalizations.
These measurements should be understood as probes of trained operator geometry rather than complete summaries of model behavior.


\section{Conclusion}
\label{sec:conclusion}

We provided a geometric, operator-level identification of self-attention by modeling token representations as a vector field over the token graph and showing that \textbf{single-head attention} is exactly a connection walk with constant transport, while \textbf{multi-head attention} is exactly a single edge-dependent connection propagation step whose effective transport is an attention-gated mixture of head transports. Empirically, across trained decoder-only Transformers from 124M to 8B parameters, we observe consistent geometric signatures predicted by the theory: effective walks stabilize across depth, effective transports approach scaled isometries in middle layers, and both phenomena strengthen as model scales. Looking forward, this connection walk perspective suggests geometry-aware objectives (energy and holonomy regularization), constraint-preserving parameterizations (near-orthogonal or group-valued transports, inverse consistency, symmetrizable walks), and diagnostic-driven pruning or routing of heads and edges, as well as extensions to tensor-valued fibers on grids and space-time graphs for vision and video, and a deeper analysis of depth dynamics via continuous and ADR limits for non-reversible, non-isometric operators.





\section*{Impact Statement}

This paper presents work whose goal is to advance the field of Machine
Learning. There are many potential societal consequences of our work, none
of which we feel must be specifically highlighted here.

\clearpage
\appendix
\section{Connection Laplacians: Continuous Background}
\label{sec:conn-laplacians}
Let $(\mathcal{M},g)$ be a Riemannian manifold and let $E\!\to\!\mathcal{M}$ be a rank-$d$ vector bundle with a metric-compatible connection $\nabla$.
For a section $X\in\Gamma(E)$, the (rough) \emph{connection Laplacian} is
\[
  \Delta_{\nabla}X \;=\; \nabla^*\nabla X \;=\; -\,\mathrm{tr}_g\,\nabla^2 X .
\]
It generates diffusion of vector fields via the heat equation
\[
  \partial_t X(t,\cdot) \;=\; -\,\Delta_{\nabla} X(t,\cdot), \qquad X(0,\cdot)=X_0 .
\]
The corresponding \emph{heat kernel} is a family of linear maps $K_t(x,y):E_y\!\to\!E_x$ with
$X(t,x)=\int_{\mathcal M}K_t(x,y)\,X_0(y)\, \mathrm{dvol}_g(y)$ and the semigroup
$K_{t+s}(x,z)=\int K_t(x,y)K_s(y,z)\,\mathrm{dvol}_g(y)$.
For small $t$, $K_t(x,y)$ concentrates near geodesics and factors (to leading order) into a scalar Gaussian weight times the \emph{parallel transport} along the geodesic.
Parallel transport $\mathsf{P}_\gamma:E_{\gamma(0)}\!\to\!E_{\gamma(1)}$ is the path-ordered exponential of $\nabla$, and its closed-loop product (holonomy) encodes curvature.

Row-stochastic $A$ provides the \emph{heat (diffusion) weights}; the one-step propagator is $I-\Lcrw=\mathcal{T}(A,O)$.
Discrete $k$-step diffusion is $\mathcal{T}^k$, whose entries sum over paths (random walk weights $\prod A$ times cumulative transport $O_{\gamma}$).
In continuous time, $e^{-t\Lcrw}$ is the vector-valued heat kernel on the graph.

\clearpage

\section{Complete Evolution Plots}
\label{app:full_plots}

We provide a detailed layerwise breakdown for each model. For every model, we present two side-by-side plots:
\begin{itemize}[topsep=0pt, itemsep=0pt, parsep=0pt]
    \item \textbf{Left:} Geometric Drift ($d_A$) and Frobenius Norm ($\|\Aeff\|_F$).
    \item \textbf{Right:} Transport Orthogonality Deviation ($\delta_{\text{orth}}$) and Scale Factor ($\mu_{\text{scale}}$).
\end{itemize}
This layout supports visual comparison between the stabilization of the token graph topology and the transport-orthogonality diagnostic.

\subsection{GPT-2 Series (from 124M to 1.5B)}

The GPT-2 family (Figure~\ref{fig:app_gpt2_all}) illustrates the effect of scaling within an older architecture. 
\textbf{Observation:} GPT-2 Small (Row 1) exhibits significant variance in orthogonality deviation ($\delta_{\text{orth}} \approx 0.15$) and slower geometric stabilization. As we scale up to GPT-2 XL (Row 4), the curves become noticeably smoother, and the orthogonality deviation drops to $\approx 0.06$, resembling modern large models. This suggests a scale-associated trend toward lower transport deviation in this controlled sample.

\begin{figure}[h!]
    \centering
    \begin{subfigure}{0.48\textwidth}
        \includegraphics[width=\linewidth]{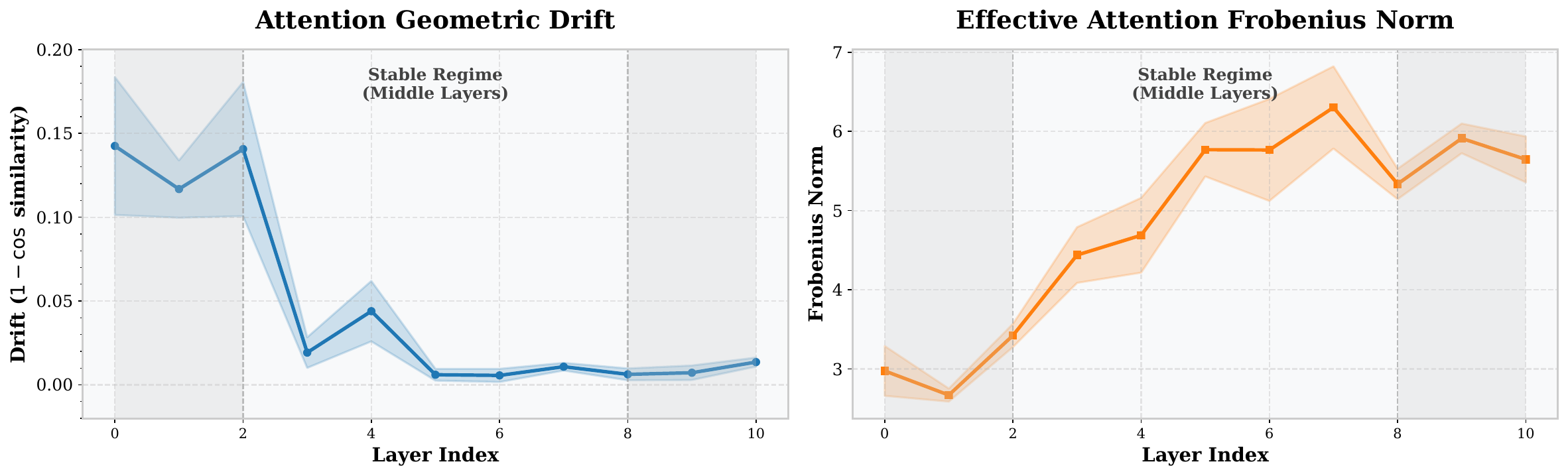}
        \caption{GPT-2 Small: Geometric Drift}
    \end{subfigure}
    \hfill
    \begin{subfigure}{0.48\textwidth}
        \includegraphics[width=\linewidth]{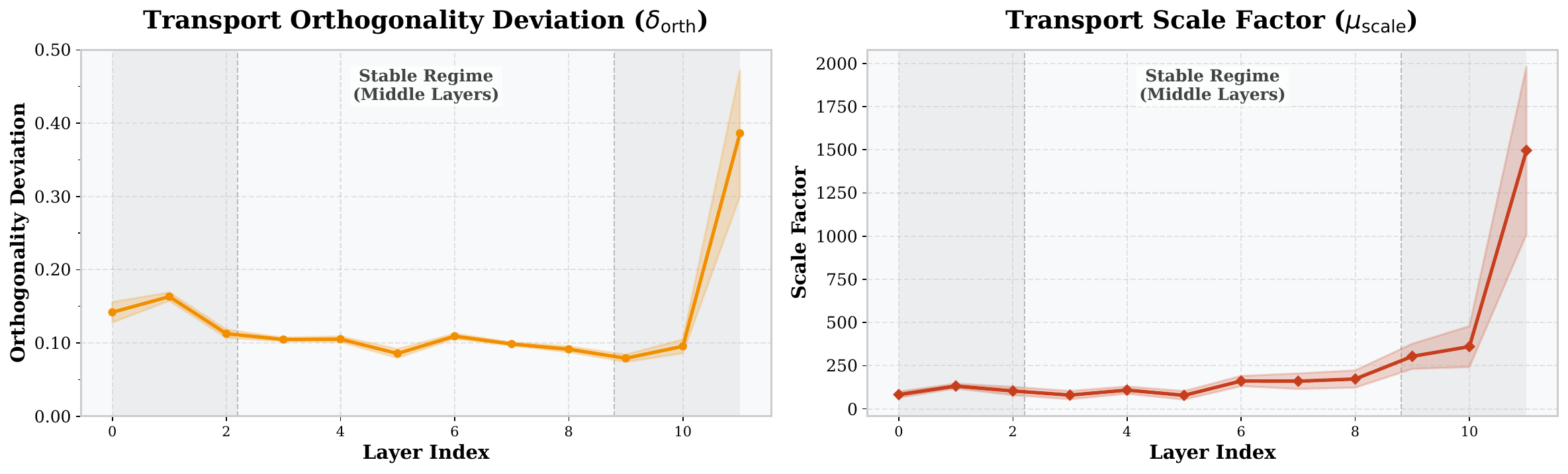}
        \caption{GPT-2 Small: Orthogonality}
    \end{subfigure}
    
    \vspace{0.2cm}

    \begin{subfigure}{0.48\textwidth}
        \includegraphics[width=\linewidth]{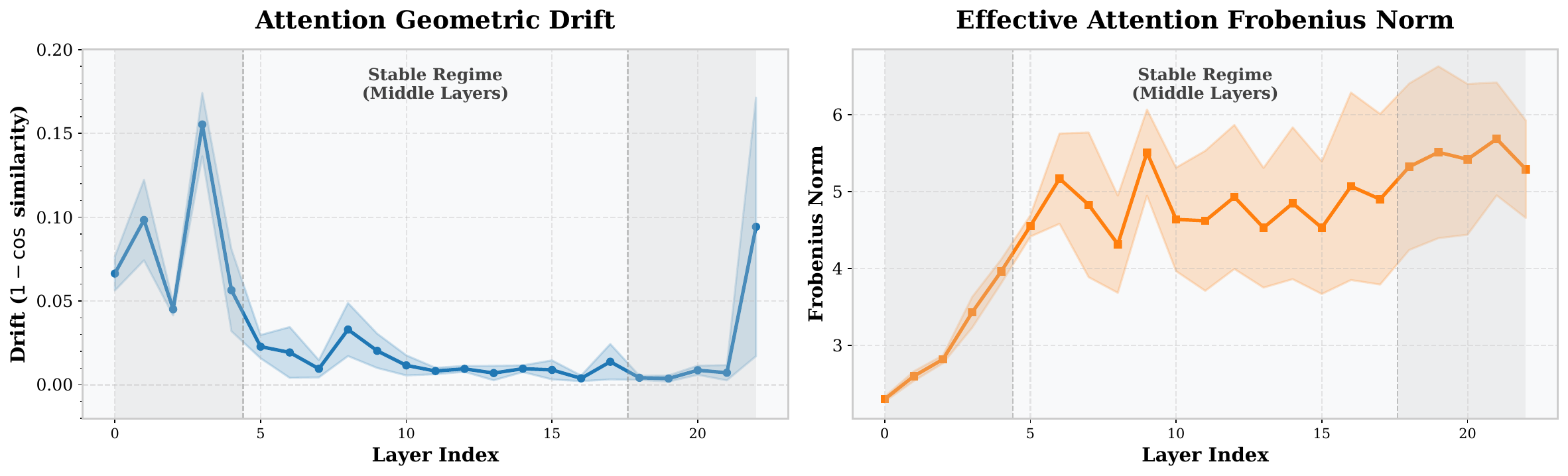}
        \caption{GPT-2 Medium: Geometric Drift}
    \end{subfigure}
    \hfill
    \begin{subfigure}{0.48\textwidth}
        \includegraphics[width=\linewidth]{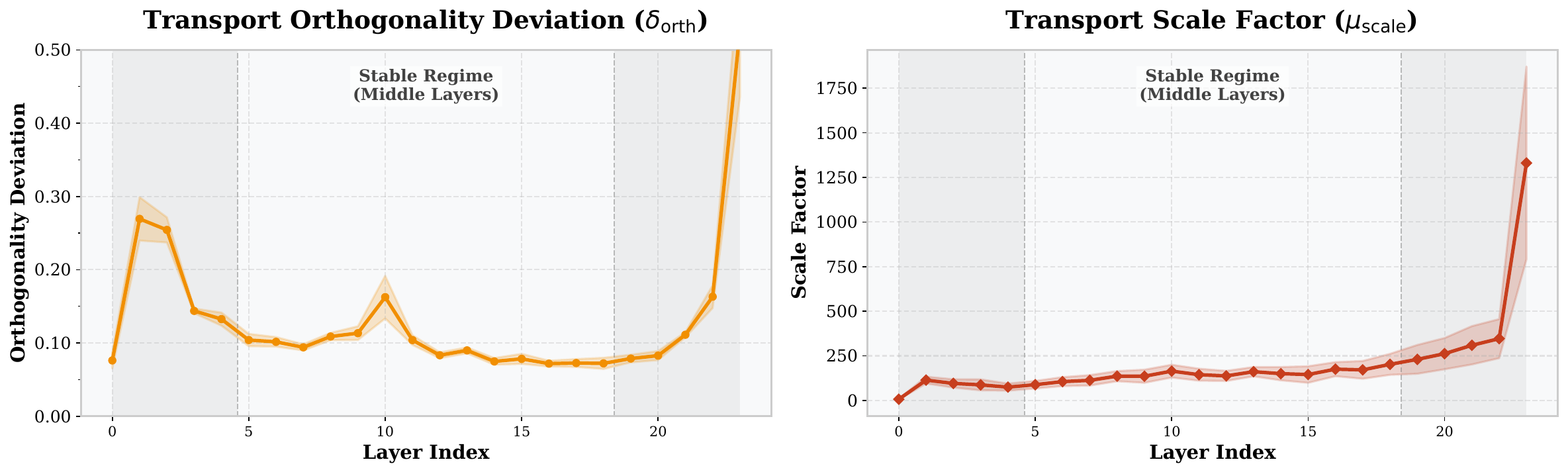}
        \caption{GPT-2 Medium: Orthogonality}
    \end{subfigure}
    
    \vspace{0.2cm}

    \begin{subfigure}{0.48\textwidth}
        \includegraphics[width=\linewidth]{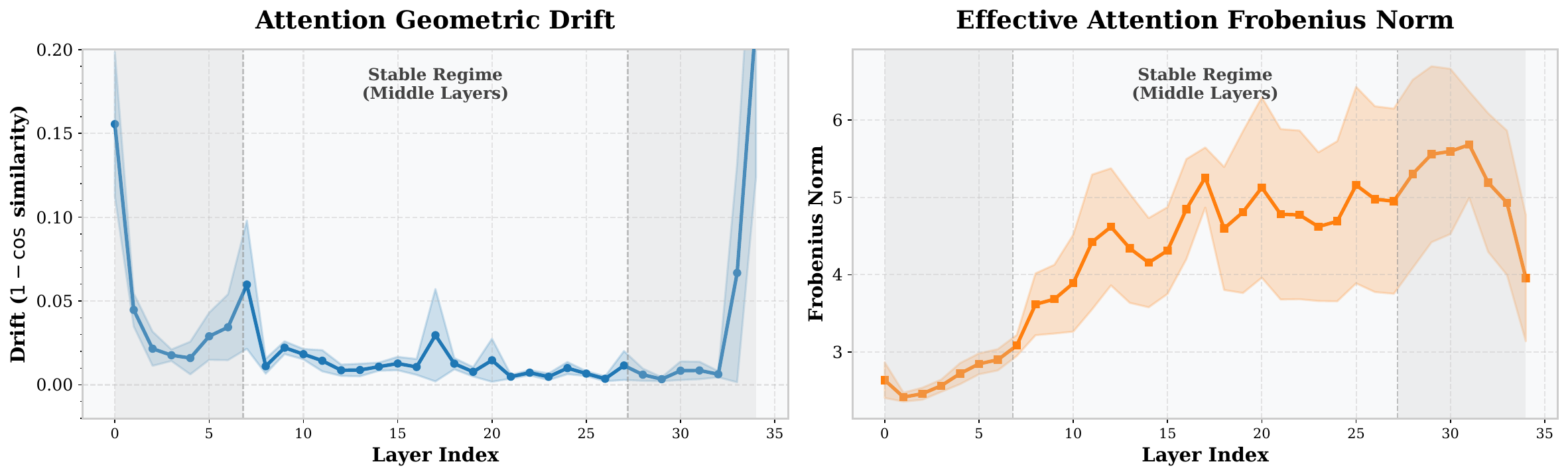}
        \caption{GPT-2 Large: Geometric Drift}
    \end{subfigure}
    \hfill
    \begin{subfigure}{0.48\textwidth}
        \includegraphics[width=\linewidth]{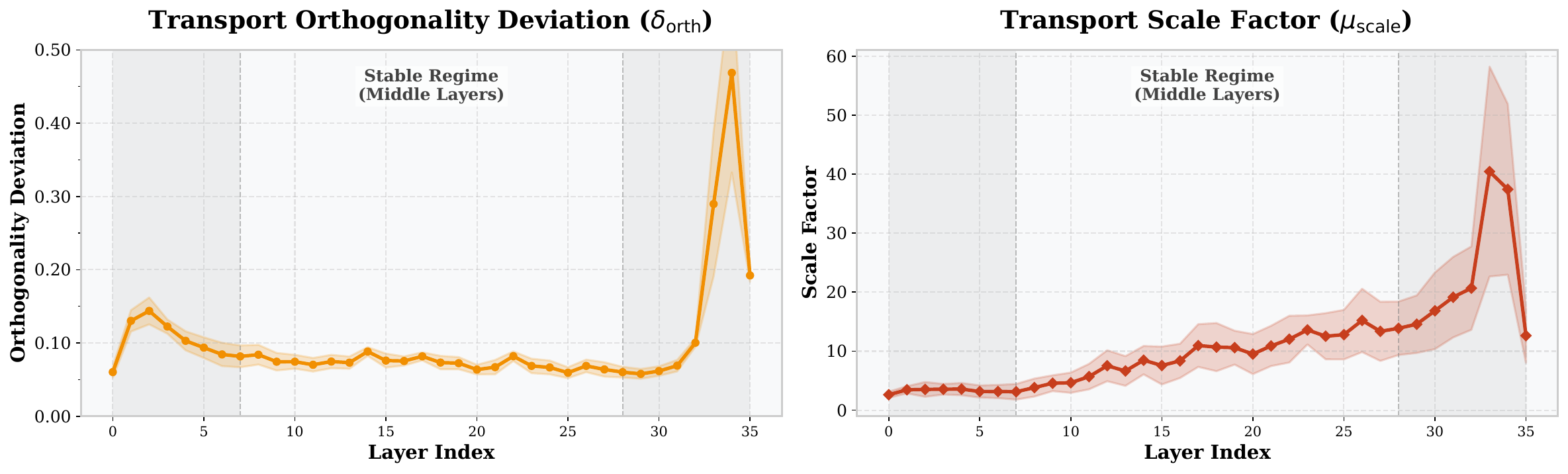}
        \caption{GPT-2 Large: Orthogonality}
    \end{subfigure}

    \vspace{0.2cm}

    \begin{subfigure}{0.48\textwidth}
        \includegraphics[width=\linewidth]{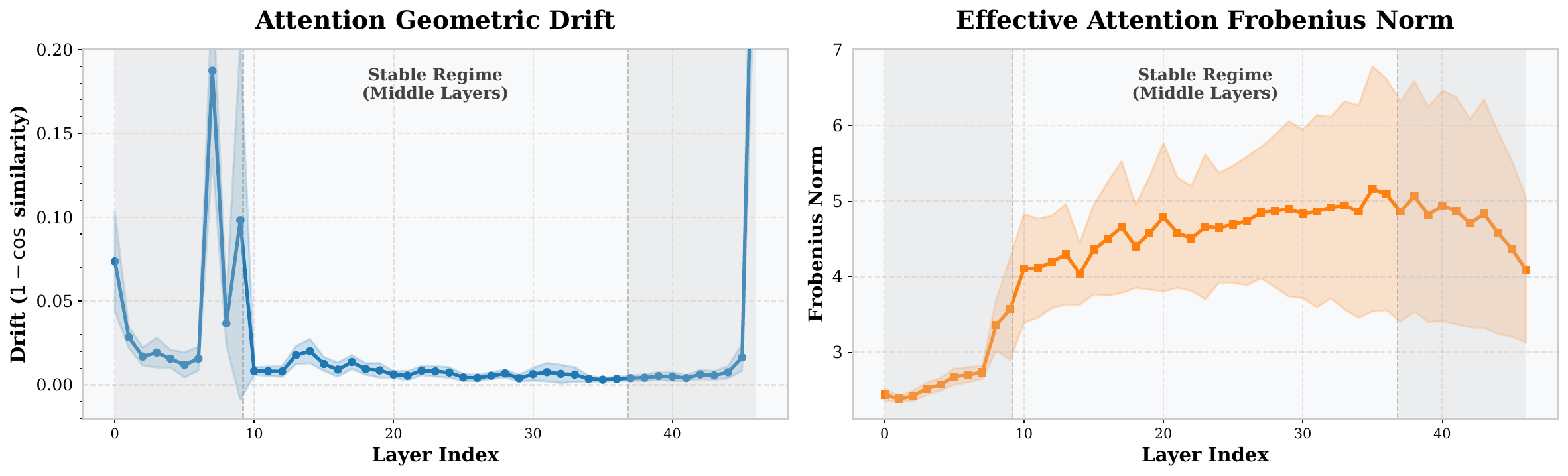}
        \caption{GPT-2 XL: Geometric Drift}
    \end{subfigure}
    \hfill
    \begin{subfigure}{0.48\textwidth}
        \includegraphics[width=\linewidth]{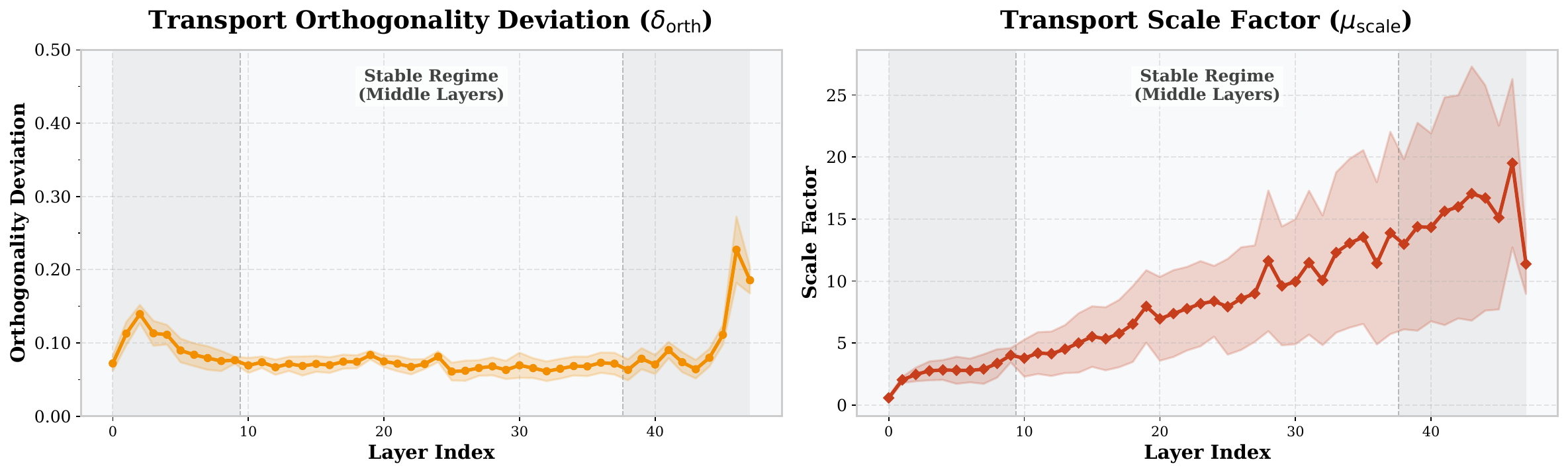}
        \caption{GPT-2 XL: Orthogonality}
    \end{subfigure}

    \caption{\textbf{GPT-2 Series Evolution.} Comparison of geometric stability (Left) and transport orthogonality (Right) across four model sizes. The larger variants in this controlled family show smoother curves and lower orthogonality deviation.}
    \label{fig:app_gpt2_all}
\end{figure}

\subsection{Qwen Series (2.5 \& 3)}

Figure~\ref{fig:app_qwen_all} shows the results for the Qwen2.5 and Qwen3 families. 
\textbf{Observation:} These modern architectures display a characteristic ``U-shape" in drift, with low drift in the deep middle layers. Qwen3-8B achieves low orthogonality deviation under the scaled-isometry diagnostic.

\begin{figure}[h!]
    \centering
    \begin{subfigure}{0.48\textwidth}
        \includegraphics[width=\linewidth]{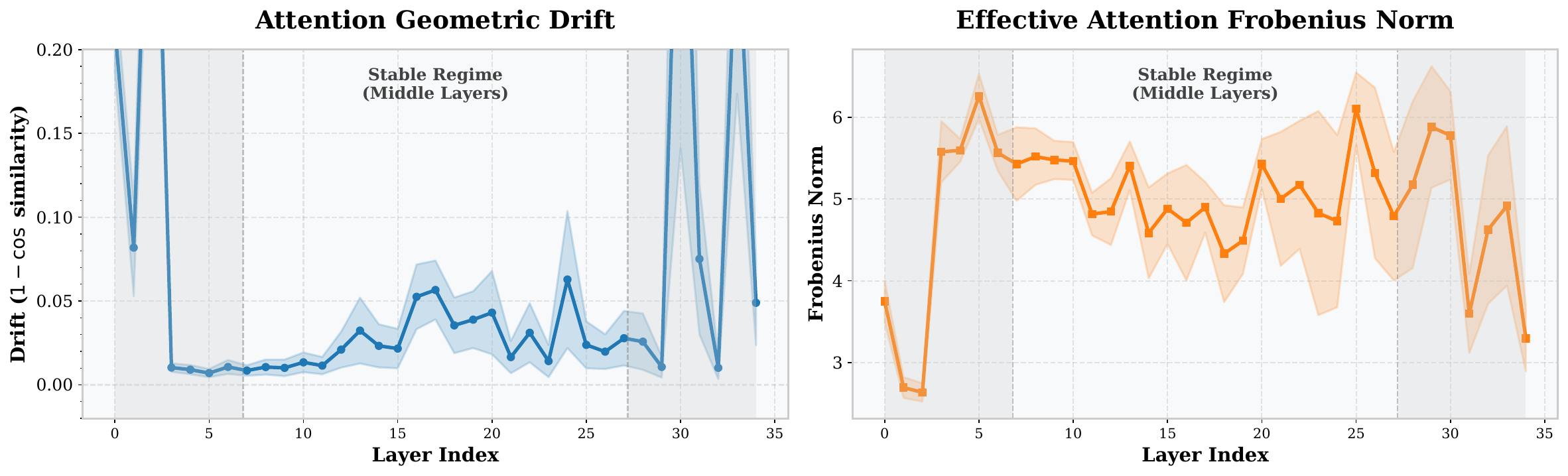}
        \caption{Qwen2.5-3B: Geometric Drift}
    \end{subfigure}
    \hfill
    \begin{subfigure}{0.48\textwidth}
        \includegraphics[width=\linewidth]{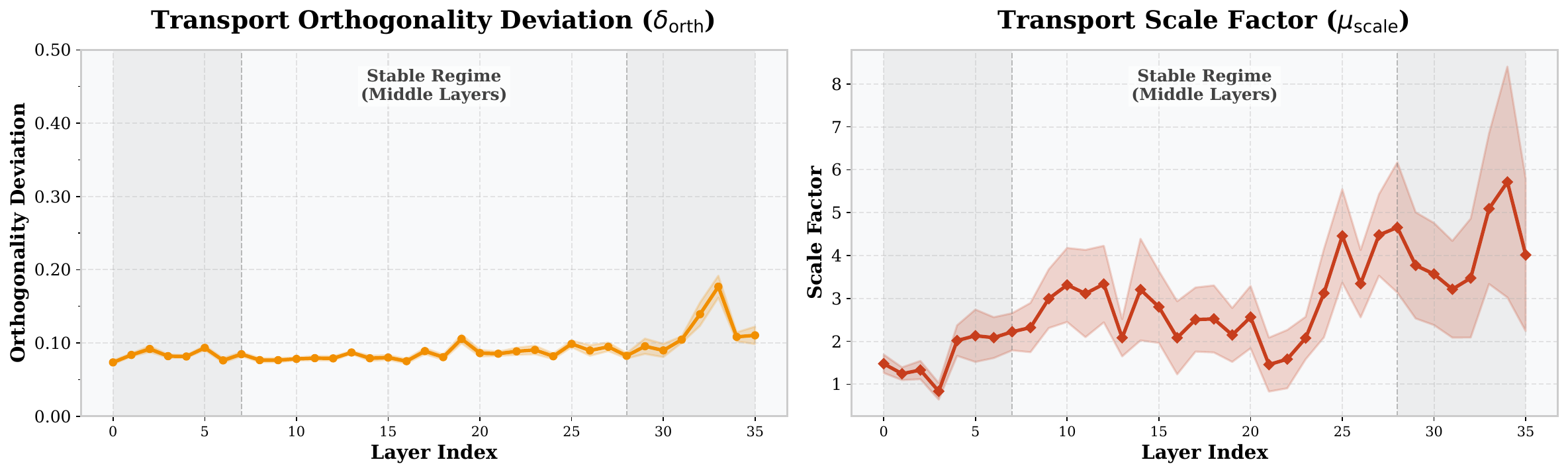}
        \caption{Qwen2.5-3B: Orthogonality}
    \end{subfigure}
    
    \vspace{0.2cm}

    \begin{subfigure}{0.48\textwidth}
        \includegraphics[width=\linewidth]{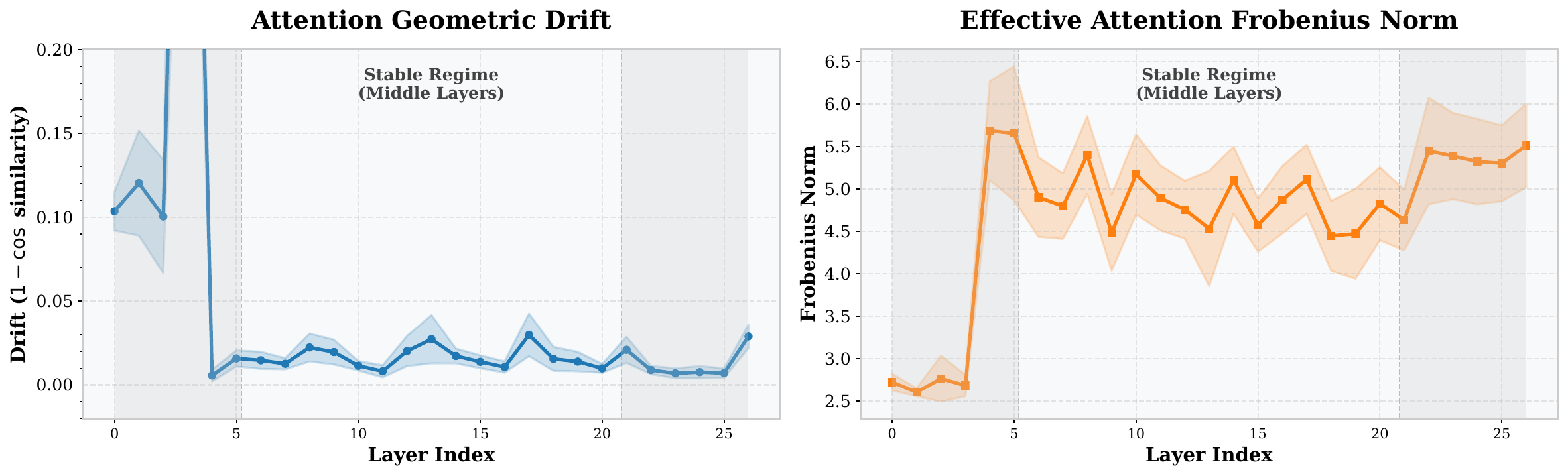}
        \caption{Qwen2.5-7B: Geometric Drift}
    \end{subfigure}
    \hfill
    \begin{subfigure}{0.48\textwidth}
        \includegraphics[width=\linewidth]{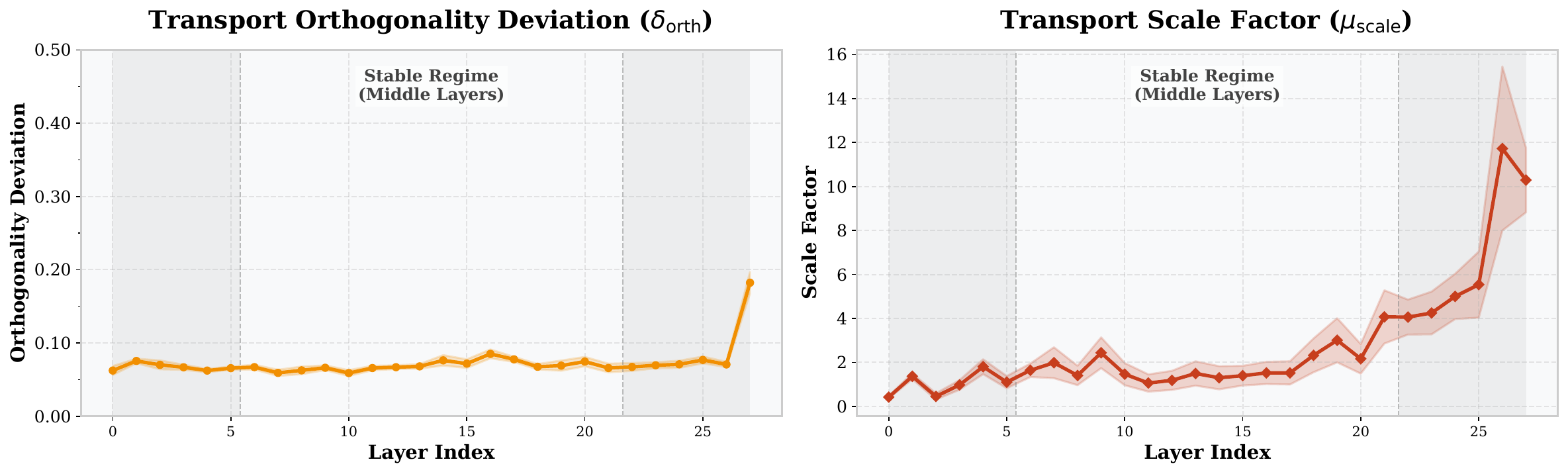}
        \caption{Qwen2.5-7B: Orthogonality}
    \end{subfigure}

    \vspace{0.2cm}

    \begin{subfigure}{0.48\textwidth}
        \includegraphics[width=\linewidth]{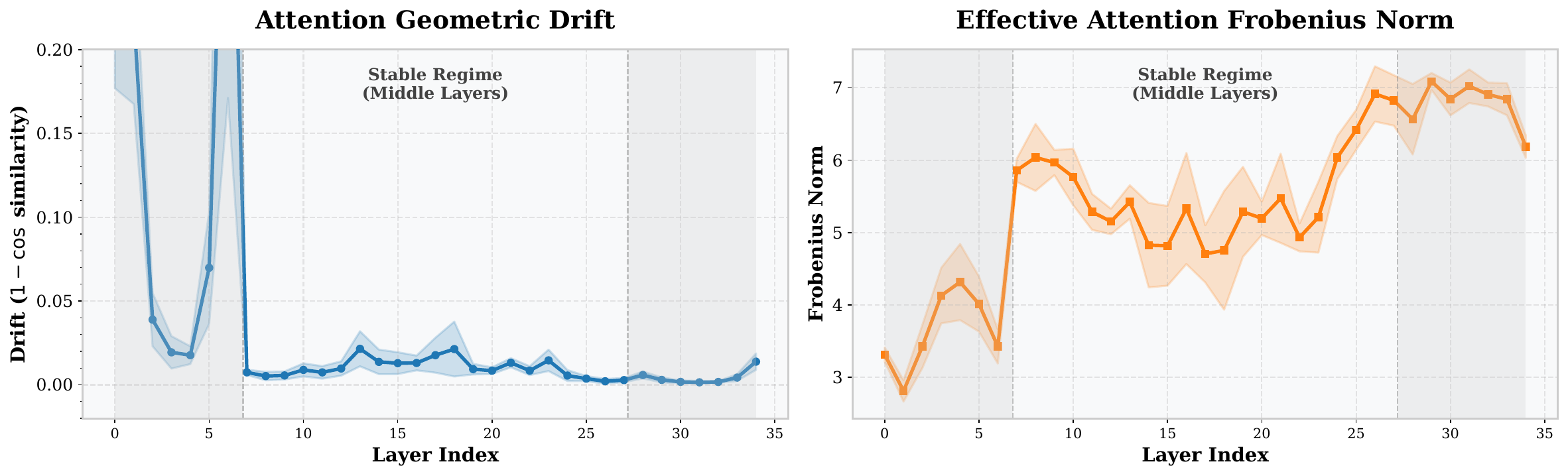}
        \caption{Qwen3-4B: Geometric Drift}
    \end{subfigure}
    \hfill
    \begin{subfigure}{0.48\textwidth}
        \includegraphics[width=\linewidth]{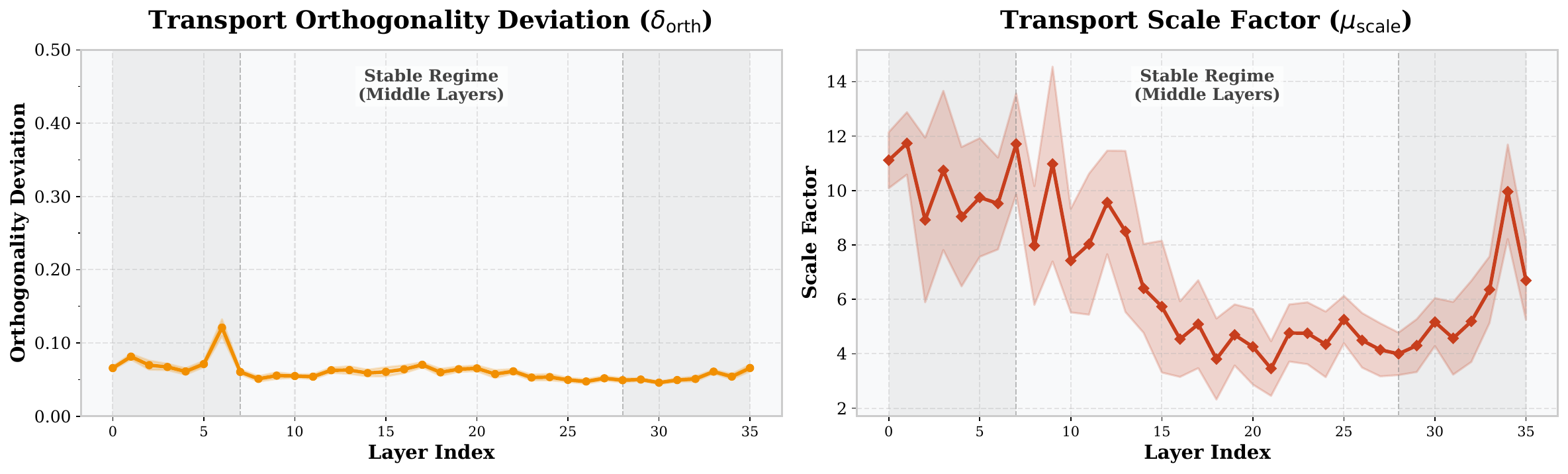}
        \caption{Qwen3-4B: Orthogonality}
    \end{subfigure}
    
    \vspace{0.2cm}

    \begin{subfigure}{0.48\textwidth}
        \includegraphics[width=\linewidth]{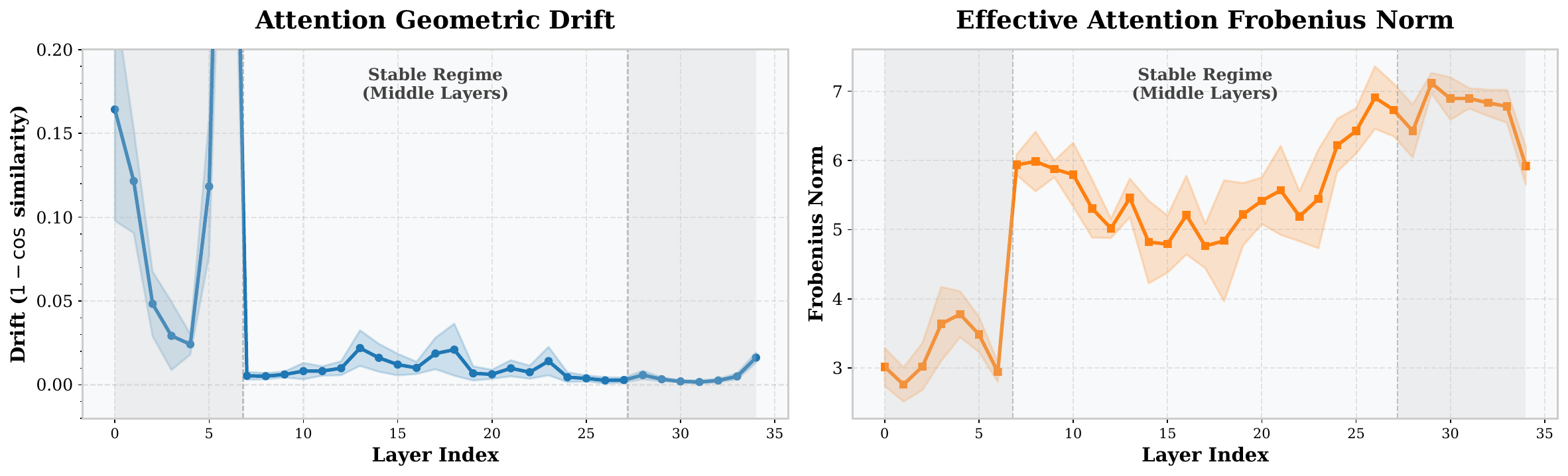}
        \caption{Qwen3-8B: Geometric Drift}
    \end{subfigure}
    \hfill
    \begin{subfigure}{0.48\textwidth}
        \includegraphics[width=\linewidth]{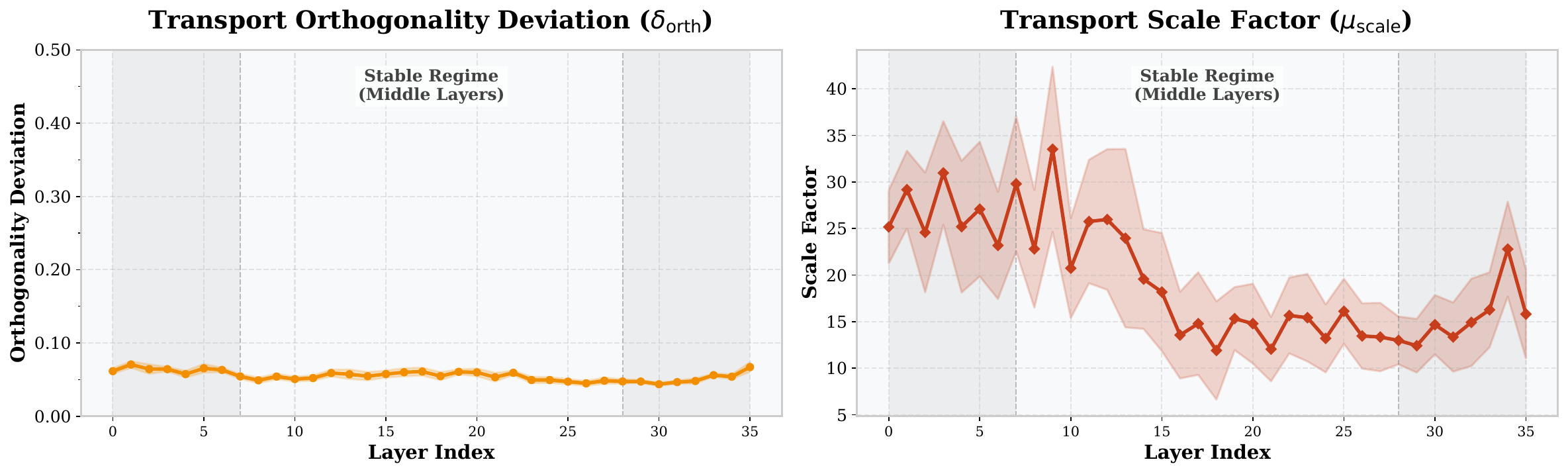}
        \caption{Qwen3-8B: Orthogonality}
    \end{subfigure}

    \caption{\textbf{Qwen Series Evolution.} Both Qwen2.5 and Qwen3 show low middle-layer drift under the geometric diagnostic.}
    \label{fig:app_qwen_all}
\end{figure}

\subsection{Llama Series (3.2 \& 3)}

Figure~\ref{fig:app_llama_all} displays the Llama-3 family results. 
\textbf{Observation:} Llama-3-8B has among the lowest transport deviations in our evaluation, with $\delta_{\text{orth}}$ consistently below $0.06$ for the majority of layers. The drop in drift (Left) and stabilization of orthogonality (Right) are particularly clear in this model.

\begin{figure}[h!]
    \centering
    \begin{subfigure}{0.48\textwidth}
        \includegraphics[width=\linewidth]{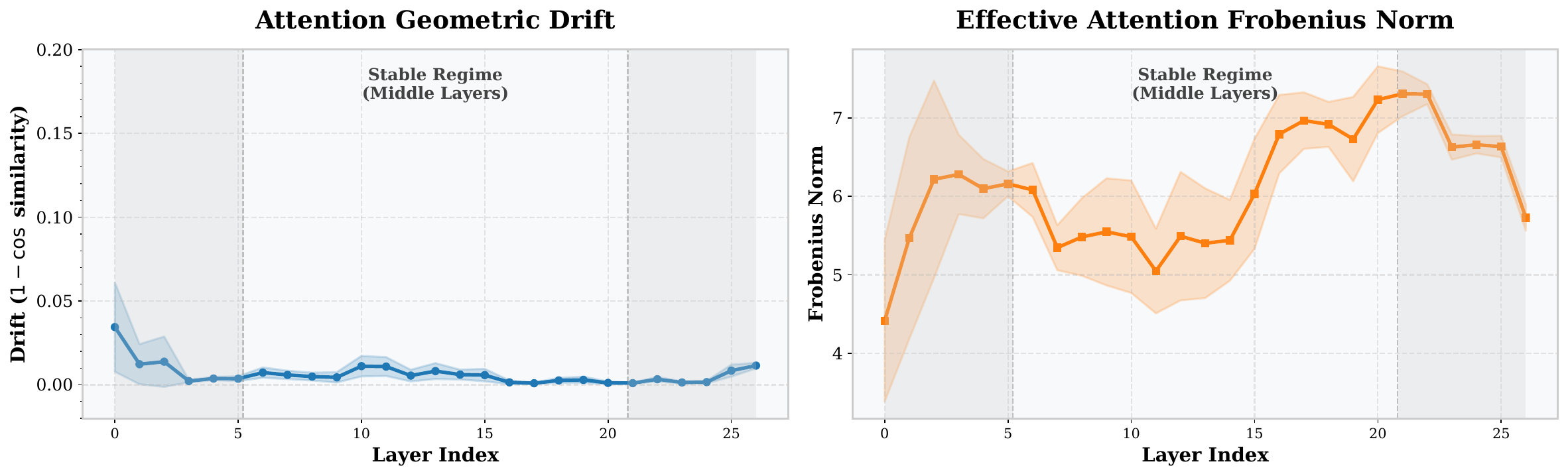}
        \caption{Llama-3.2-3B: Geometric Drift}
    \end{subfigure}
    \hfill
    \begin{subfigure}{0.48\textwidth}
        \includegraphics[width=\linewidth]{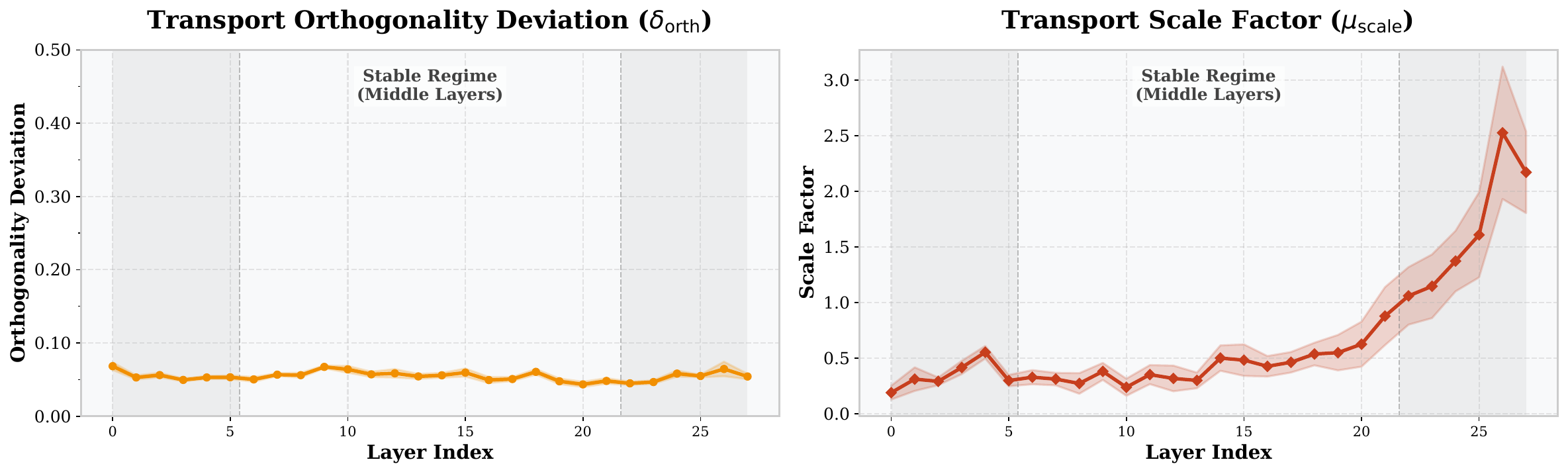}
        \caption{Llama-3.2-3B: Orthogonality}
    \end{subfigure}
    
    \vspace{0.2cm}

    \begin{subfigure}{0.48\textwidth}
        \includegraphics[width=\linewidth]{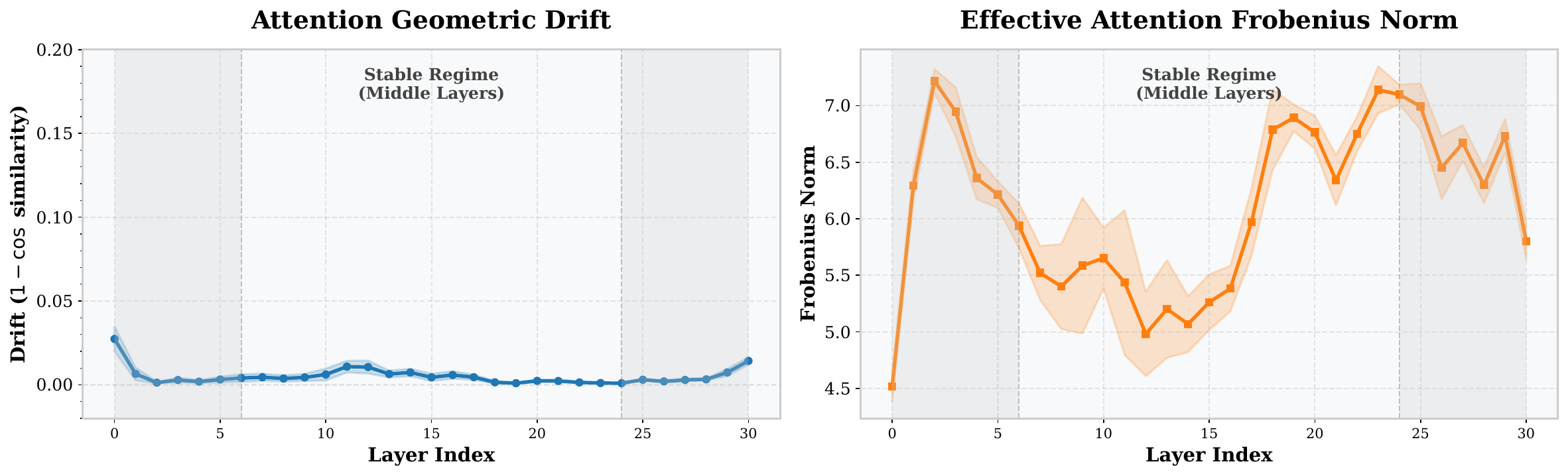}
        \caption{Llama-3-8B: Geometric Drift}
    \end{subfigure}
    \hfill
    \begin{subfigure}{0.48\textwidth}
        \includegraphics[width=\linewidth]{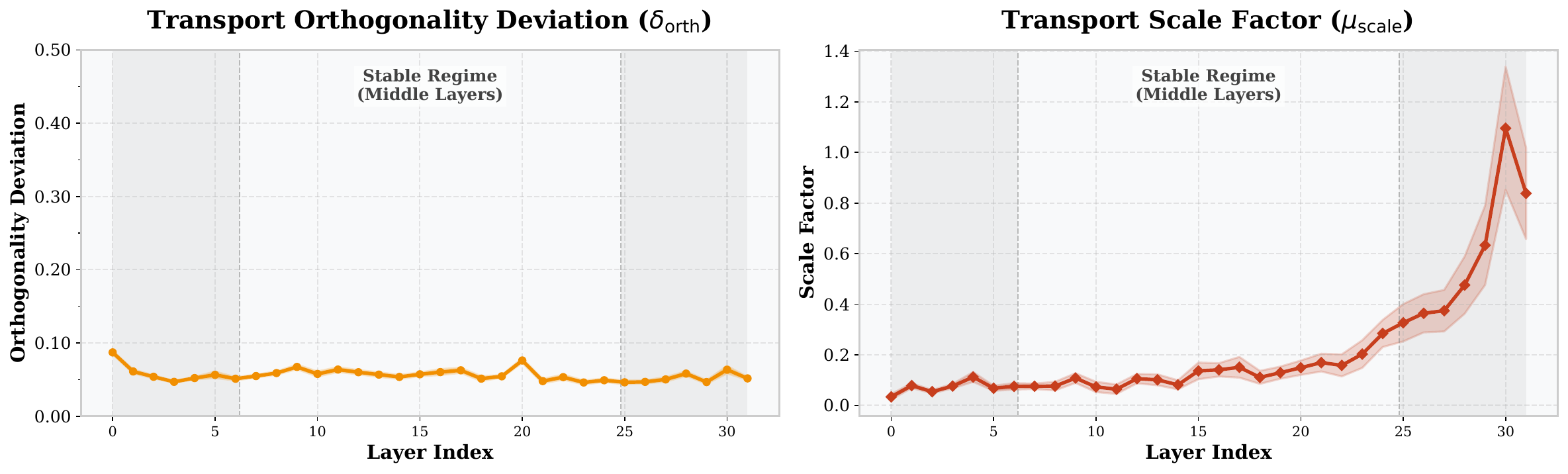}
        \caption{Llama-3-8B: Orthogonality}
    \end{subfigure}

    \caption{\textbf{Llama Series Evolution.} Llama models show low drift and low transport deviation under the proposed diagnostics.}
    \label{fig:app_llama_all}
\end{figure}

\subsection{Encoder-Only BERT}
\label{app:bert_encoder}

Figure~\ref{fig:bert_encoder_diagnostics} reports the same operator diagnostics for BERT-base-uncased.
Unless otherwise stated, the setup uses WikiText-2 test samples and a length-64 truncation.
For BERT, we remove \texttt{[CLS]}, \texttt{[SEP]}, and \texttt{[PAD]} positions before computing token-graph diagnostics, and we use the full valid bidirectional mask rather than a causal lower-triangular support.
\textbf{Observation:} BERT shows high adjacent-layer similarity overall, but its drift profile is less cleanly stabilized than the decoder-only GPT-2 probe. Its transport orthogonality deviation decreases in later layers, suggesting that the scaled-isometry diagnostic is not limited to decoder-only self-attention while remaining architecture-dependent.

\begin{figure}[H]
\centering
    \begin{subfigure}[t]{0.49\linewidth}
        \centering
        \includegraphics[width=\linewidth]{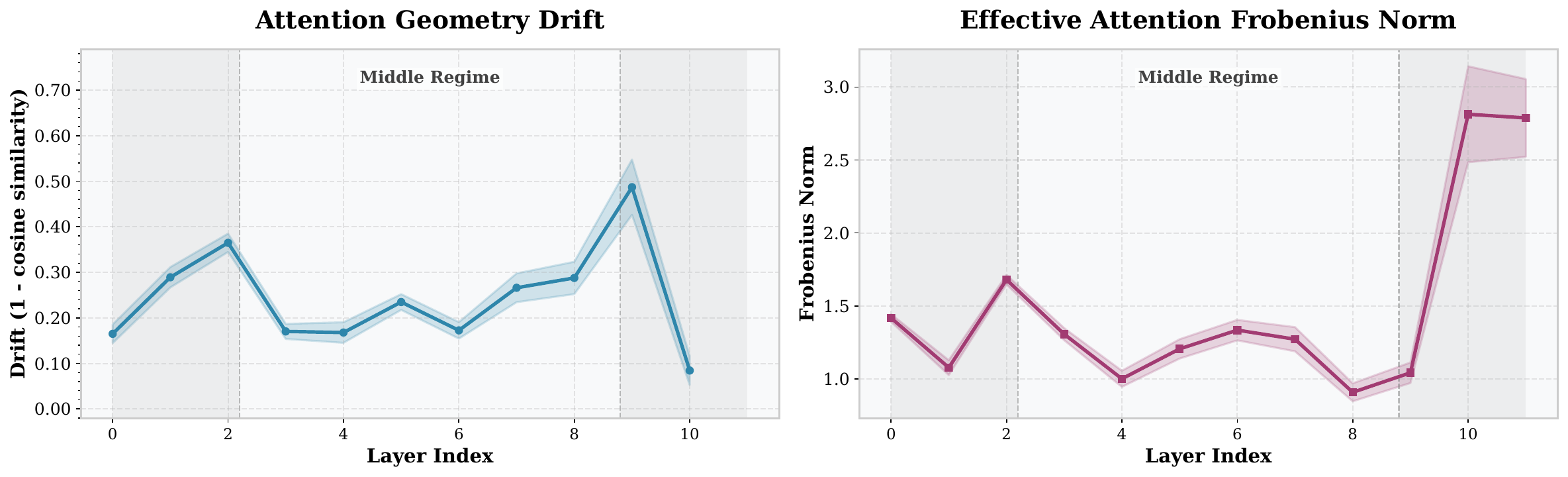}
        \caption{Effective-walk drift and $\|\Aeff\|_F$.}
        \label{fig:bert_aeff}
    \end{subfigure}
    \hfill
    \begin{subfigure}[t]{0.49\linewidth}
        \centering
        \includegraphics[width=\linewidth]{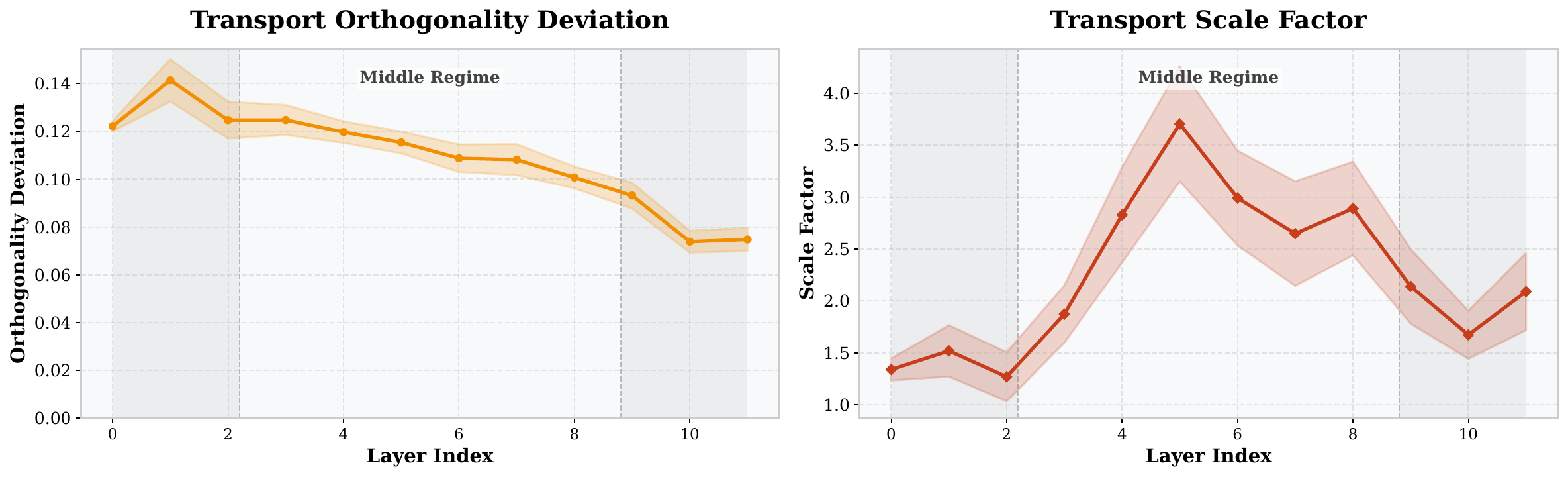}
        \caption{Transport orthogonality and scale.}
        \label{fig:bert_orthogonality}
    \end{subfigure}
\caption{BERT-base-uncased operator diagnostics using the full valid bidirectional non-special-token mask.}
\label{fig:bert_encoder_diagnostics}
\end{figure}

\clearpage

\section{Statistical Analysis of Transport Matrices}
\label{app:orth_stats}

\subsection{Diagonal Dominance}

We quantitatively assess the ``scaled isometry'' hypothesis by examining the structure of the local Gram matrix $G_{ij}^\ell = {O_{ij}^{\mathrm{eff},\ell}}^\top O_{ij}^{\mathrm{eff},\ell}$. We report the mean absolute magnitude of diagonal versus off-diagonal elements across three depth regimes: \textsc{Initial} (0--20\%), \textsc{Middle} (20--80\%), and \textsc{Final} (80--100\%).

As detailed in Table~\ref{tab:diag_off_diag_full}, the transport matrices exhibit consistent \textbf{diagonal dominance} across all model families. Focusing on the stable ``Middle'' regime:
\begin{itemize}
    \item For \textbf{GPT-2 Small}, the diagonal mean ($133.8$) is approximately \textbf{15x} larger than the off-diagonal mean ($9.0$).
    \item For \textbf{Llama-3-8B}, the diagonal mean ($0.1412$) is over \textbf{20x} larger than the off-diagonal mean ($0.0068$).
\end{itemize}
Intermediate models, including the Qwen series and larger GPT-2 variants, consistently fall within this range, maintaining high diagonal-to-off-diagonal ratios (typically $>10\times$) as shown in Table~\ref{tab:diag_off_diag_full}.

This dominance indicates that the learned transport operator $O_{ij}^{\mathrm{eff}}$ is closer to a scaled or rotational map than to an arbitrary dense shear map under this aggregate diagnostic. These metrics corroborate the visual sparsity observed in Figure~\ref{fig:heatmaps}, while not ruling out non-isometric effects on individual edges.

\begin{table}[h]
\centering
\caption{Quantitative analysis of Transport Gram Matrix $G_{ij}$. We report the mean magnitude of diagonal vs. off-diagonal elements. The ``Middle'' regime (layers from 20\% to 80\% depth) highlights strong diagonal dominance (approx. 15x--20x ratio), consistent with approximate scaled-isometry under this diagnostic.}
\label{tab:diag_off_diag_full}
\begin{small}
\begin{sc}
\setlength{\tabcolsep}{5pt} 
\begin{tabular}{lcccccc}
\toprule
 & \multicolumn{2}{c}{Initial (20\%)} & \multicolumn{2}{c}{Middle (60\%)} & \multicolumn{2}{c}{Final (20\%)} \\
\cmidrule(lr){2-3} \cmidrule(lr){4-5} \cmidrule(lr){6-7}
Model & Diag & Off-Diag & Diag & Off-Diag & Diag & Off-Diag \\
\midrule
GPT-2 (124M)       & 131.98 & 6.49  & 133.81 & 9.00  & 699.14 & 19.10 \\
GPT-2 Medium (335M)& 81.15  & 4.20  & 155.00 & 10.64 & 496.23 & 15.47 \\
GPT-2 Large (774M) & 3.00   & 0.25  & 8.71   & 0.58  & 18.54  & 0.89  \\
GPT-2 XL (1.5B)   & 2.21   & 0.18  & 5.93   & 0.40  & 12.35  & 1.16  \\
\midrule
Qwen2.5-3B         & 1.59   & 0.10  & 2.53   & 0.16  & 3.66   & 0.24  \\
Qwen2.5-7B         & 1.10   & 0.06  & 1.76   & 0.09  & 7.51   & 0.46  \\
\midrule
Qwen3-4B           & 10.40  & 0.61  & 5.67   & 0.29  & 6.57   & 0.35  \\
Qwen3-8B           & 26.73  & 1.45  & 16.23  & 0.76  & 14.97  & 0.69  \\
\midrule
Llama-3.2-3B       & 0.34   & 0.01  & 0.40   & 0.02  & 1.69   & 0.08  \\
Llama-3-8B         & 0.08   & 0.004 & \textbf{0.14} & \textbf{0.007} & 0.61 & 0.03 \\
\bottomrule
\end{tabular}
\end{sc}
\end{small}
\end{table}

\clearpage

\section{Ablation Study: Sequence Length}
\label{app:ablation}

We investigate the robustness of geometric stability and orthogonality under varying context lengths.
The main text extends GPT-2 Small to $1{,}024$ tokens in Figure~\ref{fig:ablation_main}; Figure~\ref{fig:full_ablation} reports the broader model-family comparison for context lengths $L \in \{16, 32, 64, 96, 128, 256\}$.

\begin{figure}[H]
    \centering
    \begin{subfigure}[b]{0.48\linewidth}
        \centering
        \includegraphics[width=\linewidth]{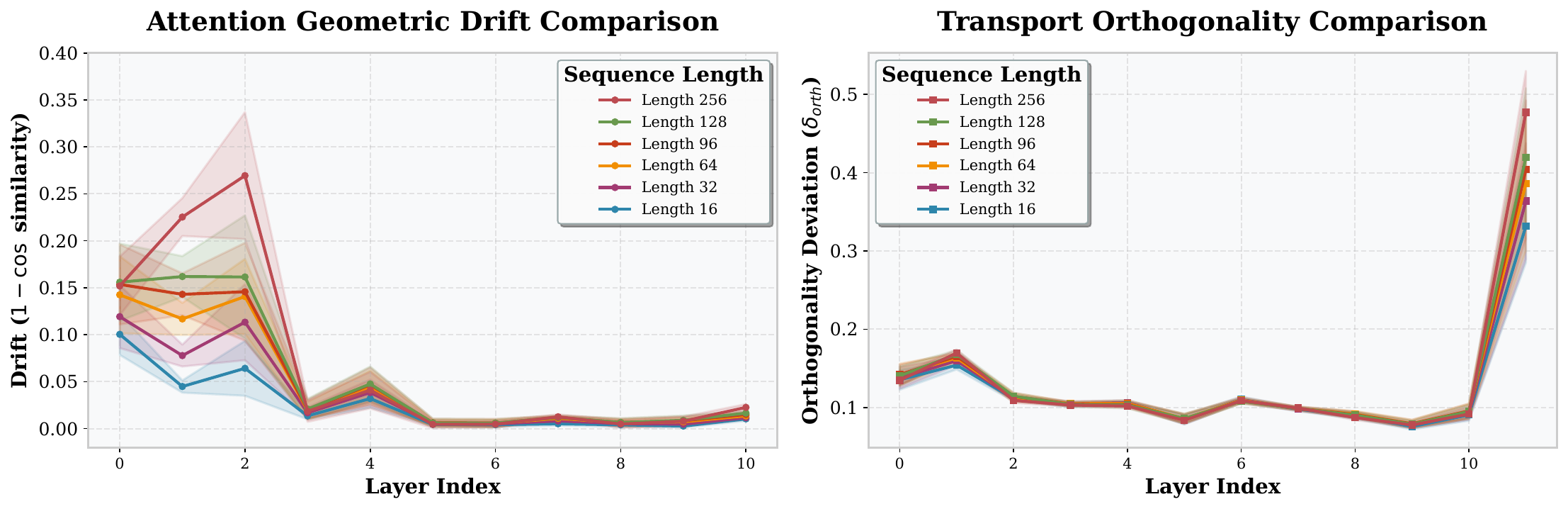} 
        \caption{GPT-2 Small}
    \end{subfigure}
    \hfill
    \begin{subfigure}[b]{0.48\linewidth}
        \centering
        \includegraphics[width=\linewidth]{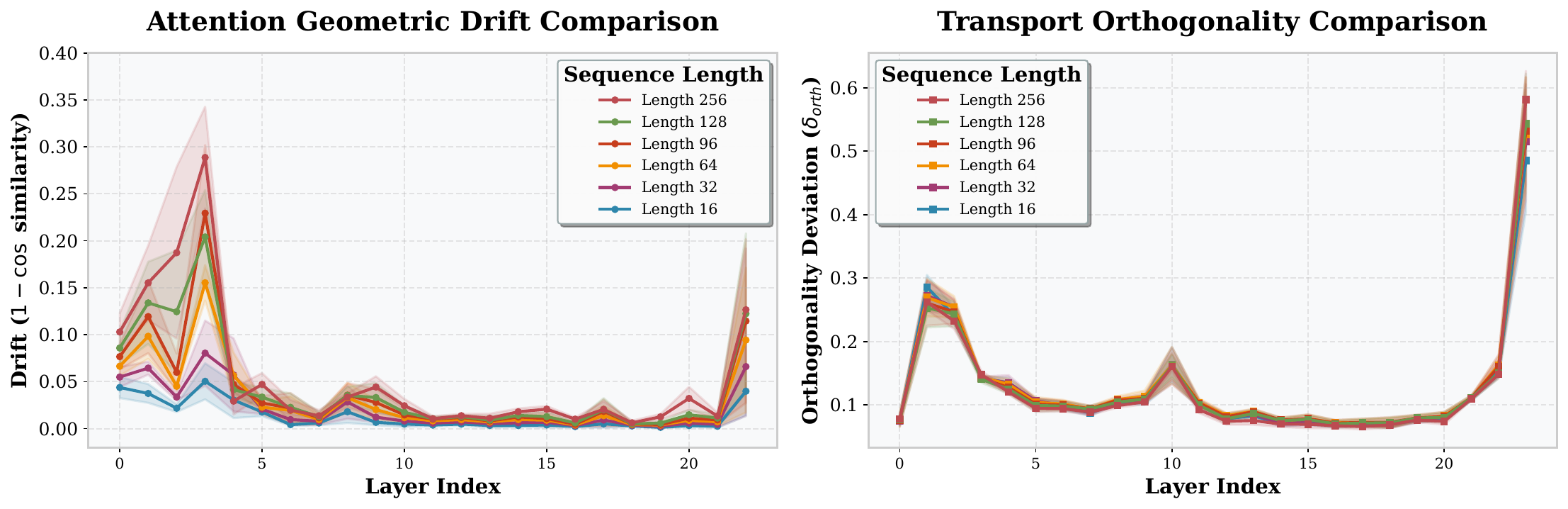}
        \caption{GPT-2 Medium}
    \end{subfigure}
    
    \vspace{0.05cm}

    \begin{subfigure}[b]{0.48\linewidth}
        \centering
        \includegraphics[width=\linewidth]{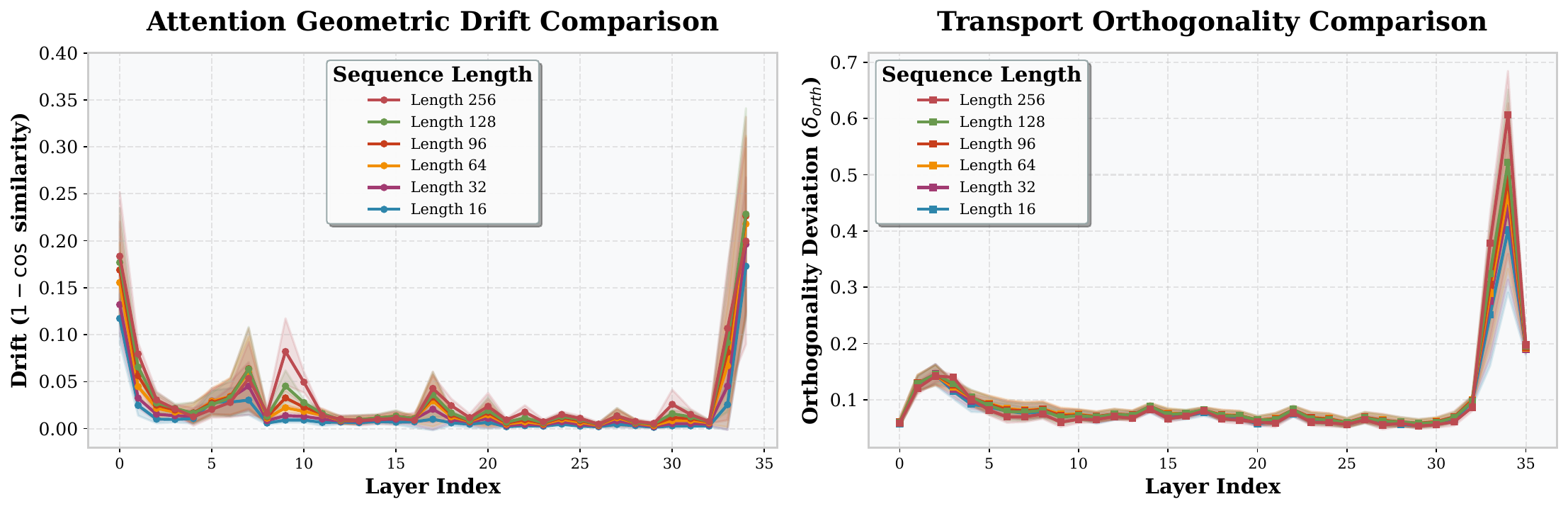} 
        \caption{GPT-2 Large}
    \end{subfigure}
    \hfill
    \begin{subfigure}[b]{0.48\linewidth}
        \centering
        \includegraphics[width=\linewidth]{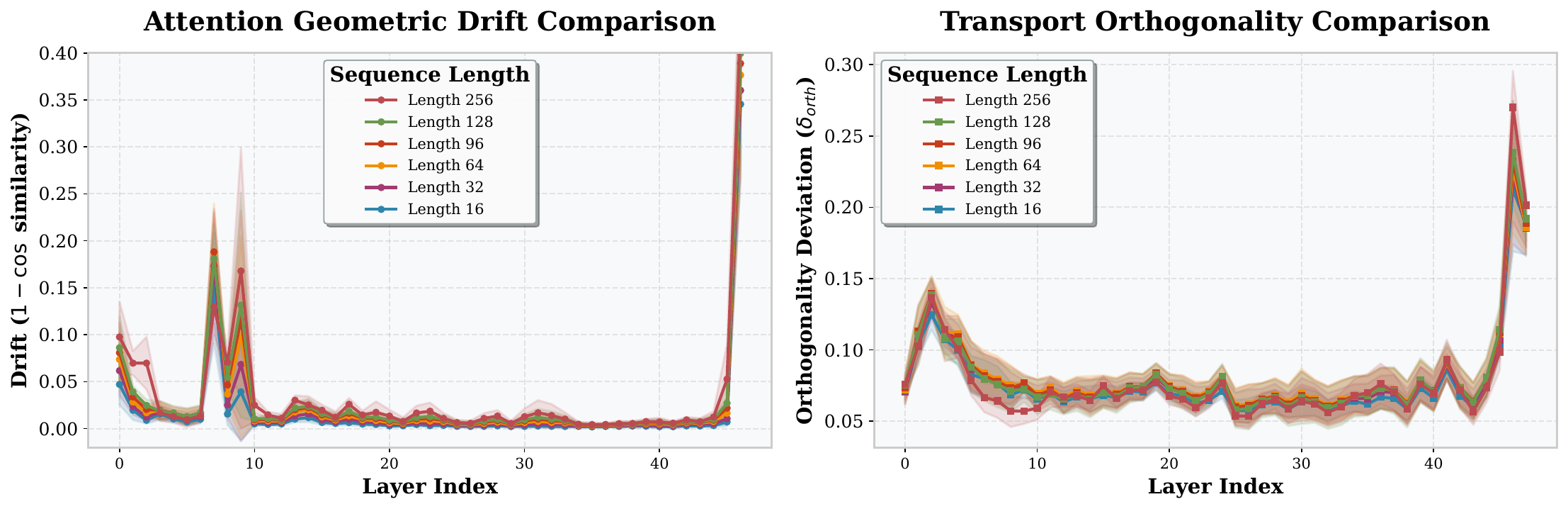}
        \caption{GPT-2 XL}
    \end{subfigure}

    \vspace{0.05cm}

    \begin{subfigure}[b]{0.48\linewidth}
        \centering
        \includegraphics[width=\linewidth]{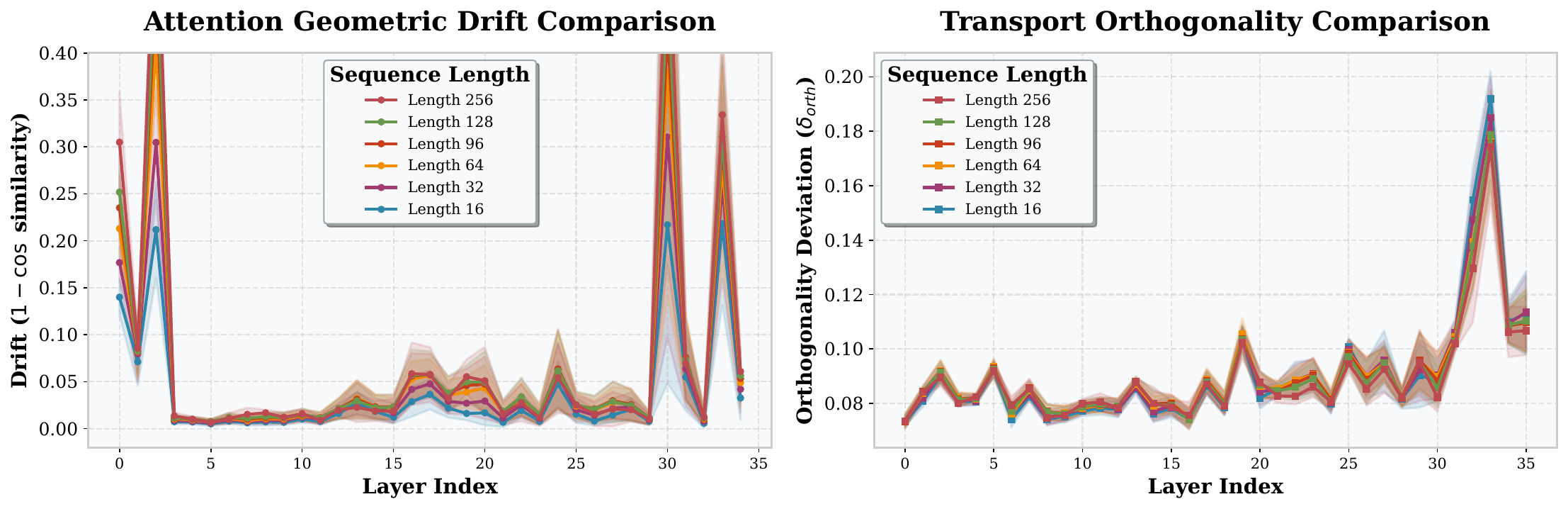}
        \caption{Qwen2.5-3B}
    \end{subfigure}
    \hfill
    \begin{subfigure}[b]{0.48\linewidth}
        \centering
        \includegraphics[width=\linewidth]{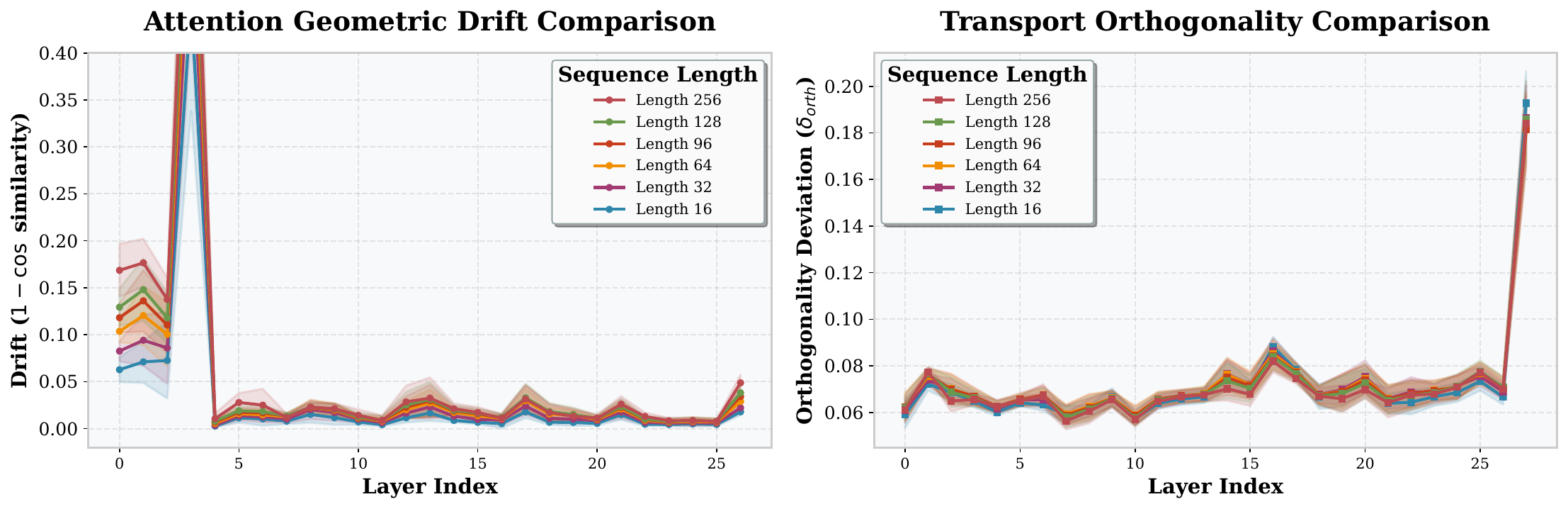}
        \caption{Qwen2.5-7B}
    \end{subfigure}
    
    \vspace{0.05cm}

    \begin{subfigure}[b]{0.48\linewidth}
        \centering
        \includegraphics[width=\linewidth]{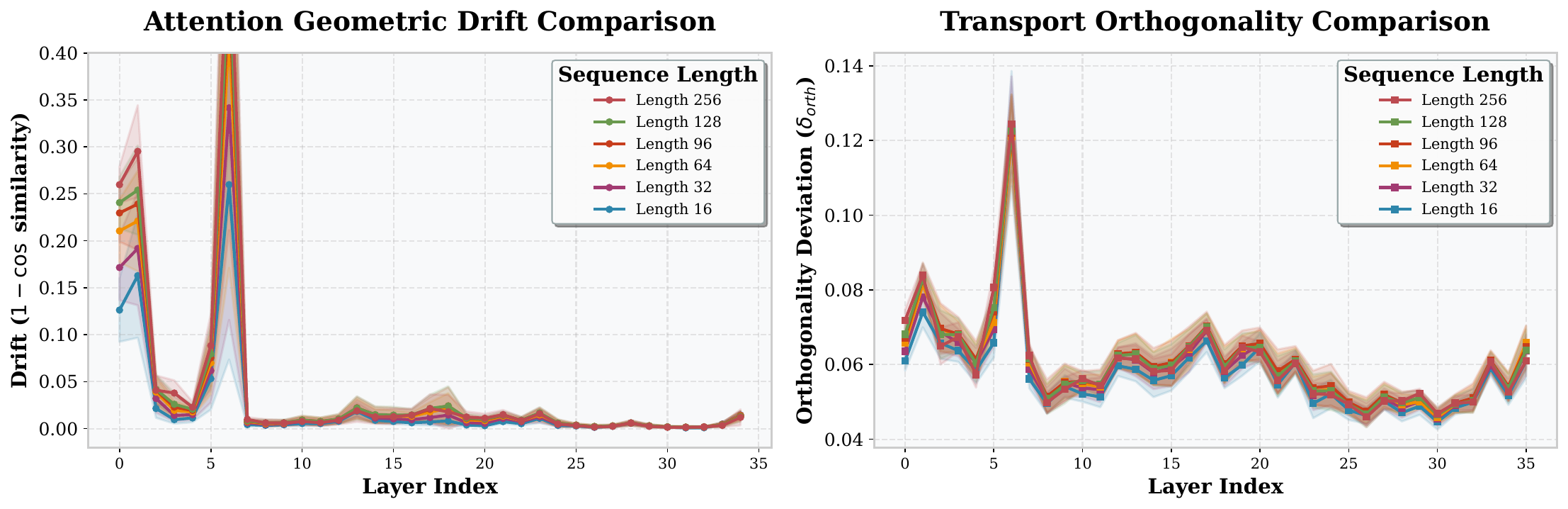}
        \caption{Qwen3-4B}
    \end{subfigure}
    \hfill
    \begin{subfigure}[b]{0.48\linewidth}
        \centering
        \includegraphics[width=\linewidth]{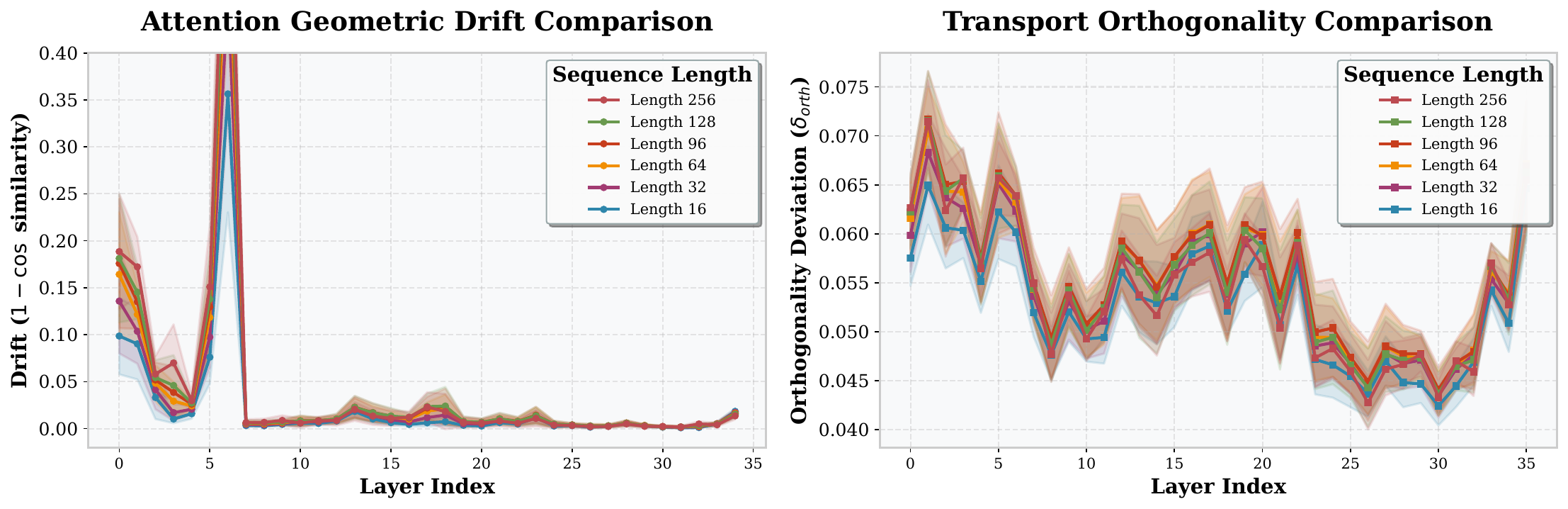}
        \caption{Qwen3-8B}
    \end{subfigure}
    
    \vspace{0.05cm}
    
    \begin{subfigure}[b]{0.48\linewidth}
        \centering
        \includegraphics[width=\linewidth]{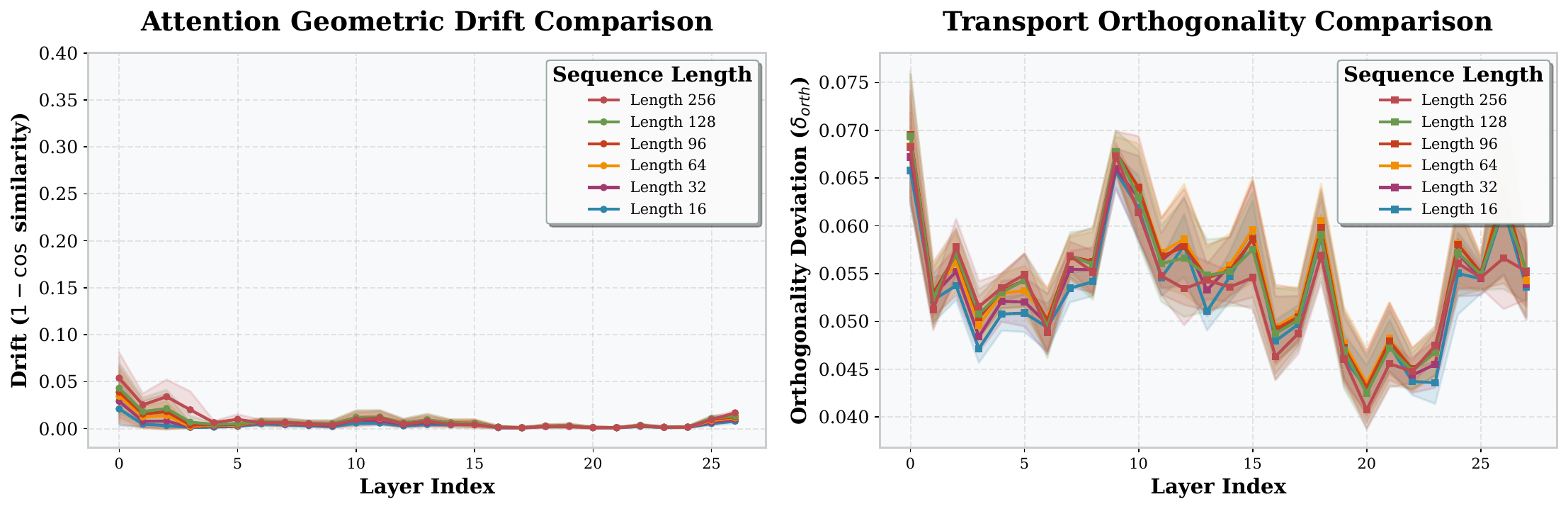}
        \caption{Llama-3.2-3B}
    \end{subfigure}
    \hfill
    \begin{subfigure}[b]{0.48\linewidth}
        \centering
        \includegraphics[width=\linewidth]{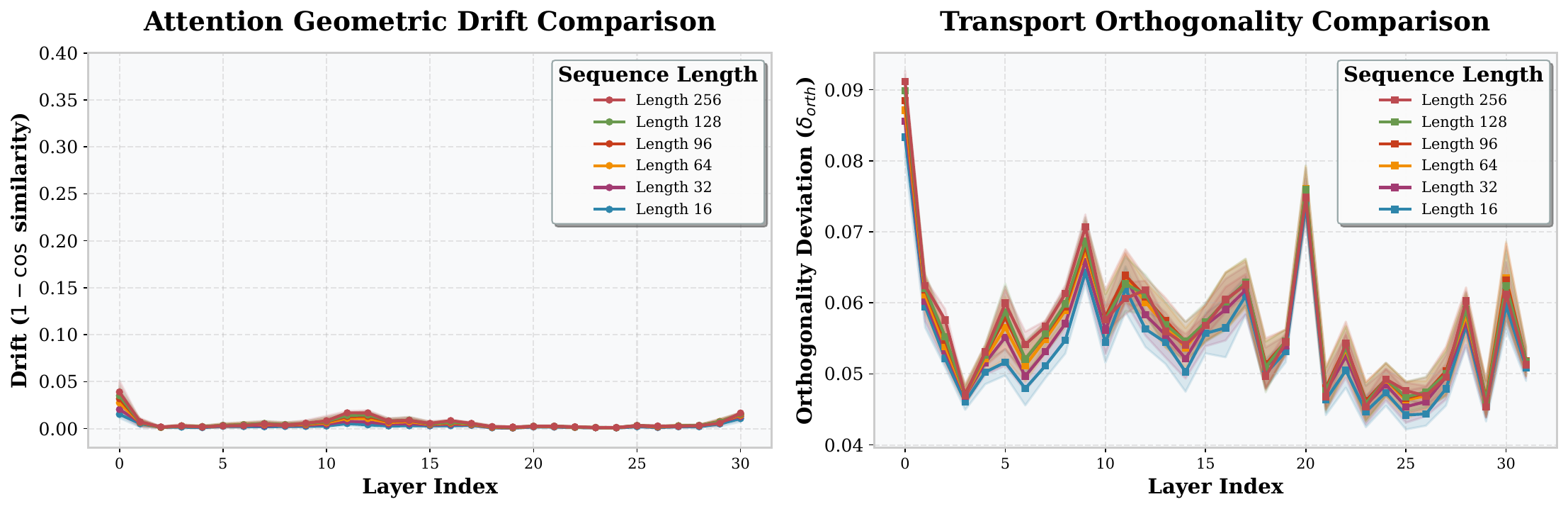}
        \caption{Llama-3-8B}
    \end{subfigure}

    \caption{\textbf{Impact of Sequence Length.} Comparing drift and orthogonality across model families for context lengths $L=16$ to $L=256$. Across the tested architectures, deeper layers often recover lower drift despite increased initial drift for longer sequences.}
    \label{fig:full_ablation}
\end{figure}

\clearpage

\section{LayerNorm and FFN Details for the Block-Level ADR View}
\label{app:block_dynamics}

This appendix expands the block-level discussion in Section~\ref{subsec:block-scope}.
The exact theorem-level statement of the paper concerns the attention sublayer, while residual connections, normalization, and FFNs determine how such attention sublayers are composed across depth.

For a pre-norm Transformer block, the schematic form is
\[
  U = X + \mathrm{Attn}(\LN(X)),
  \qquad
  Y = U + \FFN(\LN(U)).
\]
For a post-norm block, a schematic form is
\[
  U = \LN(X + \mathrm{Attn}(X)),
  \qquad
  Y = \LN(U + \FFN(U)).
\]
In both cases, attention is the component that introduces same-layer cross-token edges.
LayerNorm and FFN are tokenwise maps: their Jacobians are block diagonal over token positions, although they can be highly data-dependent inside each token fiber.

Under continuous-depth scaling, a pre-LN block is naturally modeled as a normalized-forcing evolution
\[
  \partial_t u = F_t(\LN(u)),
\]
where $F_t$ is the nonlocal connection-ADR vector field induced by attention and FFN.
A post-LN block is better modeled as a projected evolution
\[
  \partial_t u = J_{\LN}(u)F_t(u),
\]
where the infinitesimal update is projected through the normalization Jacobian.
This distinction is interpretive and does not change the exact attention-sublayer connection-walk theorem.

\clearpage

\section{Connection-Energy Diagnostic}
\label{app:connection_energy}

For a directed and generally non-isometric attention operator, the classical PSD Dirichlet-form interpretation need not hold.
Nevertheless, the following normalized nonnegative residual is useful as a diagnostic:
\begin{equation}
  E^{(\ell)}_{\mathrm{norm,off}}
  =
  \frac{\sum_{i\ne j}\Aeff^{(\ell)}(i,j)\,
  \left\|x_i^{(\ell)}-x_j^{(\ell)}\widehat O_{ij}^{(\ell)}\right\|_2^2}
  {\frac{1}{|V|}\sum_{k\in V}\|x_k^{(\ell)}\|_2^2+\epsilon},
  \qquad
  \widehat O_{ij}^{(\ell)}=\frac{O_{ij}^{\mathrm{eff},(\ell)}}{\sqrt{\mu_{\mathrm{scale}}^{(\ell)}(i,j)}}.
\end{equation}
The self-loop terms are removed so that the diagnostic focuses on cross-token transport, and the scale-normalized transport $\widehat O_{ij}$ removes the scalar factor from an approximately scaled-orthogonal map.

Figure~\ref{fig:connection_energy} reports the off-diagonal energy measurements for GPT-2 Small and BERT-base-uncased.
GPT-2 shows comparatively low normalized energy in its middle layers relative to boundary layers, while BERT shows a broad depthwise decrease.
The raw energy and off-diagonal mass panels help separate residual magnitude from attention-mass effects.

\begin{figure}[H]
\centering
\includegraphics[width=\linewidth]{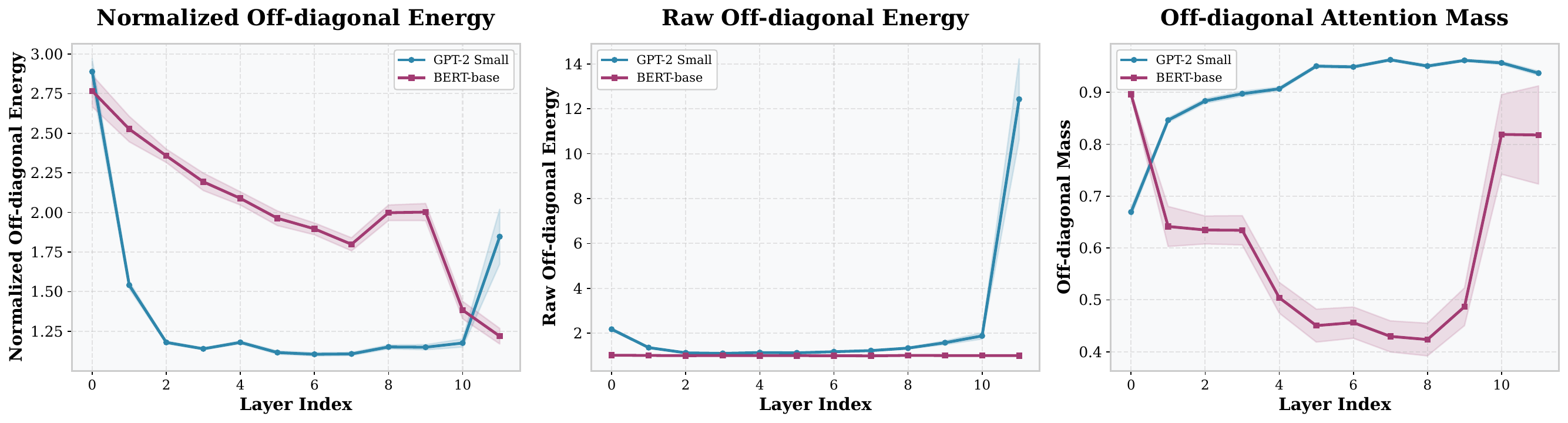}
\caption{Connection-energy diagnostics for GPT-2 Small and BERT-base-uncased. The panels report normalized off-diagonal energy, raw off-diagonal energy, and off-diagonal attention mass across layers.}
\label{fig:connection_energy}
\end{figure}

\clearpage

\section{Extended Context-Length Diagnostic}
\label{app:long_context}

The main-text ablation (Figure~\ref{fig:ablation_main}) probes context lengths up to $1{,}024$ tokens for GPT-2 Small.
Figure~\ref{fig:long_context} provides a supporting diagnostic across the same extended range.
The qualitative trend is stable across this range: longer contexts mildly increase early-layer drift and orthogonality deviation, but the deeper-layer low-drift, low-deviation profile remains visible within the tested regime.
These are descriptive observations on a single model and probe, not a causal claim about long-context competence.

\begin{figure}[H]
\centering
\includegraphics[width=\linewidth]{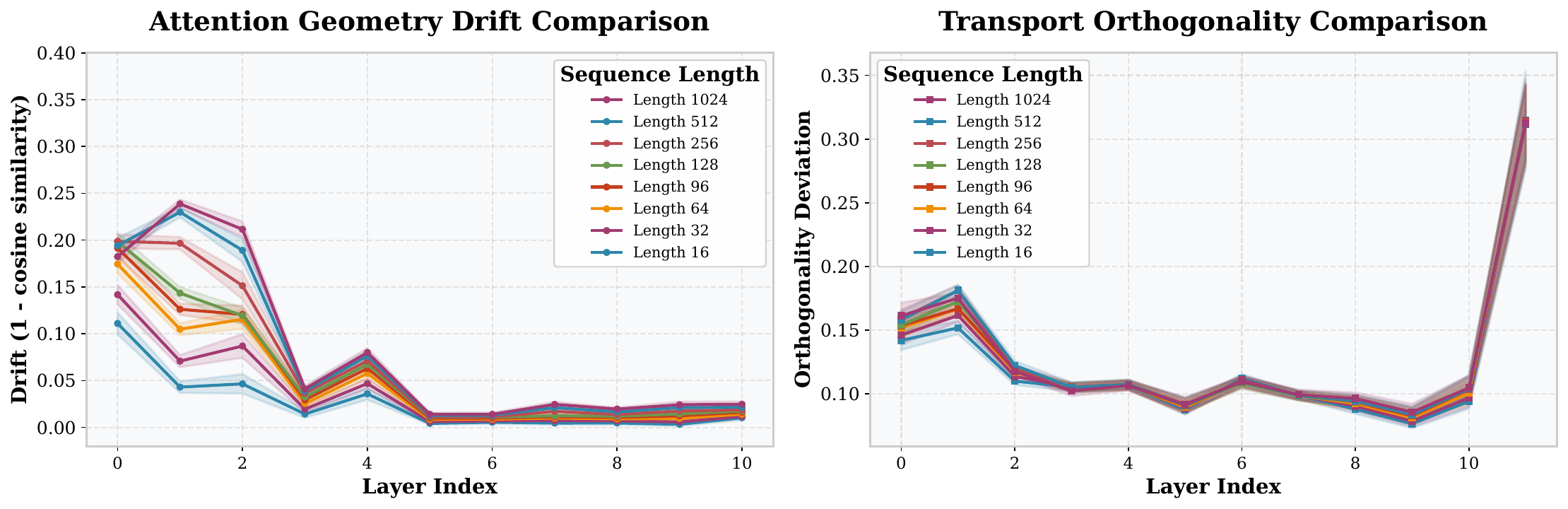}
\caption{Extended context-length diagnostic for GPT-2 Small, probing context lengths up to $1{,}024$ tokens. Longer contexts mildly increase early-layer drift and orthogonality deviation; the deeper-layer stable regime remains visible in the tested range.}
\label{fig:long_context}
\end{figure}

\clearpage

\section{Cross-Attention and Encoder--Decoder Extension}
\label{app:cross_attention}

The connection-walk formalism also extends to cross-attention by replacing the single token graph with a directed bipartite or block-structured graph.
Encoder tokens are source nodes and decoder tokens are target nodes.
For a single head with encoder states $X_E$ and decoder queries, the cross-attention update has the form
\[
  Y_D = A_{D\leftarrow E}X_E M,
\]
where $A_{D\leftarrow E}$ is row-stochastic over encoder sources for each decoder target.
Thus
\[
  (Y_D)_i = \sum_j A_{D\leftarrow E}(i,j)(X_E)_jO_{ij}.
\]
The multi-head derivation follows the same blockwise averaging argument as Theorem~\ref{thm:mha}, with an effective cross-walk $A^{\mathrm{cross}}_{\mathrm{eff}}$ and cross-edge transports $O^{\mathrm{cross,eff}}_{ij}$.
In a combined encoder-decoder block operator, cross-attention occupies an off-diagonal block from encoder-source fibers to decoder-target fibers.
A systematic empirical study of this setting remains future work.

\clearpage
\bibliography{content/references}

@inproceedings{vaswani2017attention,
  author       = {Ashish Vaswani and
                  Noam Shazeer and
                  Niki Parmar and
                  Jakob Uszkoreit and
                  Llion Jones and
                  Aidan N. Gomez and
                  Lukasz Kaiser and
                  Illia Polosukhin},
  editor       = {Isabelle Guyon and
                  Ulrike von Luxburg and
                  Samy Bengio and
                  Hanna M. Wallach and
                  Rob Fergus and
                  S. V. N. Vishwanathan and
                  Roman Garnett},
  title        = {Attention is All you Need},
  booktitle    = {Advances in Neural Information Processing Systems 30: Annual Conference
                  on Neural Information Processing Systems 2017, December 4-9, 2017,
                  Long Beach, CA, {USA}},
  pages        = {5998--6008},
  year         = {2017},
  bibsource    = {dblp computer science bibliography, https://dblp.org}
}

@inproceedings{choromanski2020performer,
  author       = {Krzysztof Marcin Choromanski and
                  Valerii Likhosherstov and
                  David Dohan and
                  Xingyou Song and
                  Andreea Gane and
                  Tam{\'{a}}s Sarl{\'{o}}s and
                  Peter Hawkins and
                  Jared Quincy Davis and
                  Afroz Mohiuddin and
                  Lukasz Kaiser and
                  David Benjamin Belanger and
                  Lucy J. Colwell and
                  Adrian Weller},
  title        = {Rethinking Attention with Performers},
  booktitle    = {9th International Conference on Learning Representations, {ICLR} 2021,
                  Virtual Event, Austria, May 3-7, 2021},
  publisher    = {OpenReview.net},
  year         = {2021},
  bibsource    = {dblp computer science bibliography, https://dblp.org}
}

@article{tay2023efficient,
  author       = {Yi Tay and
                  Mostafa Dehghani and
                  Dara Bahri and
                  Donald Metzler},
  title        = {Efficient Transformers: {A} Survey},
  journal      = {{ACM} Comput. Surv.},
  volume       = {55},
  number       = {6},
  pages        = {109:1--109:28},
  year         = {2023},
  url          = {https://doi.org/10.1145/3530811},
  doi          = {10.1145/3530811},
  bibsource    = {dblp computer science bibliography, https://dblp.org}
}

@article{wright2021nonmercer,
  author       = {Matthew A. Wright and
                  Joseph E. Gonzalez},
  title        = {Transformers are Deep Infinite-Dimensional Non-Mercer Binary Kernel
                  Machines},
  journal      = {CoRR},
  volume       = {abs/2106.01506},
  year         = {2021},
  url          = {https://arxiv.org/abs/2106.01506},
  eprinttype    = {arXiv},
  eprint       = {2106.01506},
  bibsource    = {dblp computer science bibliography, https://dblp.org}
}

@article{chowdhury2022transformerkernel,
  author       = {Sankalan Pal Chowdhury and
                  Adamos Solomou and
                  Avinava Dubey and
                  Mrinmaya Sachan},
  title        = {On Learning the Transformer Kernel},
  journal      = {CoRR},
  volume       = {abs/2110.08323},
  year         = {2021},
  url          = {https://arxiv.org/abs/2110.08323},
  eprinttype    = {arXiv},
  eprint       = {2110.08323},
  bibsource    = {dblp computer science bibliography, https://dblp.org}
}

@inproceedings{wu2024maskln,
  author       = {Xinyi Wu and
                  Amir Ajorlou and
                  Yifei Wang and
                  Stefanie Jegelka and
                  Ali Jadbabaie},
  editor       = {Amir Globersons and
                  Lester Mackey and
                  Danielle Belgrave and
                  Angela Fan and
                  Ulrich Paquet and
                  Jakub M. Tomczak and
                  Cheng Zhang},
  title        = {On the Role of Attention Masks and LayerNorm in Transformers},
  booktitle    = {Advances in Neural Information Processing Systems 38: Annual Conference
                  on Neural Information Processing Systems 2024, NeurIPS 2024, Vancouver,
                  BC, Canada, December 10 - 15, 2024},
  year         = {2024},
  bibsource    = {dblp computer science bibliography, https://dblp.org}
}

@inproceedings{ramsauer2020hopfield,
  author       = {Hubert Ramsauer and
                  Bernhard Sch{\"{a}}fl and
                  Johannes Lehner and
                  Philipp Seidl and
                  Michael Widrich and
                  Lukas Gruber and
                  Markus Holzleitner and
                  Thomas Adler and
                  David P. Kreil and
                  Michael K. Kopp and
                  G{\"{u}}nter Klambauer and
                  Johannes Brandstetter and
                  Sepp Hochreiter},
  title        = {Hopfield Networks is All You Need},
  booktitle    = {9th International Conference on Learning Representations, {ICLR} 2021,
                  Virtual Event, Austria, May 3-7, 2021},
  publisher    = {OpenReview.net},
  year         = {2021},
  bibsource    = {dblp computer science bibliography, https://dblp.org}
}

@inproceedings{yun2019gtn,
  author       = {Seongjun Yun and
                  Minbyul Jeong and
                  Raehyun Kim and
                  Jaewoo Kang and
                  Hyunwoo J. Kim},
  editor       = {Hanna M. Wallach and
                  Hugo Larochelle and
                  Alina Beygelzimer and
                  Florence d'Alch{\'{e}}{-}Buc and
                  Emily B. Fox and
                  Roman Garnett},
  title        = {Graph Transformer Networks},
  booktitle    = {Advances in Neural Information Processing Systems 32: Annual Conference
                  on Neural Information Processing Systems 2019, NeurIPS 2019, December
                  8-14, 2019, Vancouver, BC, Canada},
  pages        = {11960--11970},
  year         = {2019},
  bibsource    = {dblp computer science bibliography, https://dblp.org}
}

@article{singer2012vdm,
  title     = {Vector diffusion maps and the connection Laplacian},
  author    = {Singer, Amit and Wu, H-T},
  journal   = {Communications on pure and applied mathematics},
  volume    = {65},
  number    = {8},
  pages     = {1067--1144},
  year      = {2012},
  publisher = {Wiley Online Library}
}

@article{bandeira2013cheeger,
  author       = {Afonso S. Bandeira and
                  Amit Singer and
                  Daniel A. Spielman},
  title        = {A Cheeger Inequality for the Graph Connection Laplacian},
  journal      = {{SIAM} J. Matrix Anal. Appl.},
  volume       = {34},
  number       = {4},
  pages        = {1611--1630},
  year         = {2013},
  url          = {https://doi.org/10.1137/120875338},
  doi          = {10.1137/120875338},
  bibsource    = {dblp computer science bibliography, https://dblp.org}
}

@inproceedings{cohen2019gauge,
  author       = {Taco Cohen and
                  Maurice Weiler and
                  Berkay Kicanaoglu and
                  Max Welling},
  editor       = {Kamalika Chaudhuri and
                  Ruslan Salakhutdinov},
  title        = {Gauge Equivariant Convolutional Networks and the Icosahedral {CNN}},
  booktitle    = {Proceedings of the 36th International Conference on Machine Learning,
                  {ICML} 2019, 9-15 June 2019, Long Beach, California, {USA}},
  series       = {Proceedings of Machine Learning Research},
  volume       = {97},
  pages        = {1321--1330},
  publisher    = {{PMLR}},
  year         = {2019},
  bibsource    = {dblp computer science bibliography, https://dblp.org}
}

@article{bronstein2021gdl,
  author       = {Michael M. Bronstein and
                  Joan Bruna and
                  Taco Cohen and
                  Petar Velickovic},
  title        = {Geometric Deep Learning: Grids, Groups, Graphs, Geodesics, and Gauges},
  journal      = {CoRR},
  volume       = {abs/2104.13478},
  year         = {2021},
  url          = {https://arxiv.org/abs/2104.13478},
  eprinttype    = {arXiv},
  eprint       = {2104.13478},
  bibsource    = {dblp computer science bibliography, https://dblp.org}
}

@inproceedings{kreuzer2021rethinking,
  author       = {Devin Kreuzer and
                  Dominique Beaini and
                  William L. Hamilton and
                  Vincent L{\'{e}}tourneau and
                  Prudencio Tossou},
  editor       = {Marc'Aurelio Ranzato and
                  Alina Beygelzimer and
                  Yann N. Dauphin and
                  Percy Liang and
                  Jennifer Wortman Vaughan},
  title        = {Rethinking Graph Transformers with Spectral Attention},
  booktitle    = {Advances in Neural Information Processing Systems 34: Annual Conference
                  on Neural Information Processing Systems 2021, NeurIPS 2021, December
                  6-14, 2021, virtual},
  pages        = {21618--21629},
  year         = {2021},
  bibsource    = {dblp computer science bibliography, https://dblp.org}
}

@article{farooq2025nonlinearhopfield,
  author       = {Ahmed Farooq},
  title        = {A Framework for Non-Linear Attention via Modern Hopfield Networks},
  journal      = {CoRR},
  volume       = {abs/2506.11043},
  year         = {2025},
  url          = {https://doi.org/10.48550/arXiv.2506.11043},
  doi          = {10.48550/ARXIV.2506.11043},
  eprinttype    = {arXiv},
  eprint       = {2506.11043},
  bibsource    = {dblp computer science bibliography, https://dblp.org}
}

@inproceedings{babiloni2020tesa,
  author       = {Francesca Babiloni and
                  Ioannis Marras and
                  Gregory G. Slabaugh and
                  Stefanos Zafeiriou},
  title        = {{TESA:} Tensor Element Self-Attention via Matricization},
  booktitle    = {2020 {IEEE/CVF} Conference on Computer Vision and Pattern Recognition,
                  {CVPR} 2020, Seattle, WA, USA, June 13-19, 2020},
  pages        = {13942--13951},
  publisher    = {Computer Vision Foundation / {IEEE}},
  year         = {2020},
  doi          = {10.1109/CVPR42600.2020.01396},
  bibsource    = {dblp computer science bibliography, https://dblp.org}
}

@inproceedings{ma2019tensorized,
  author       = {Xindian Ma and
                  Peng Zhang and
                  Shuai Zhang and
                  Nan Duan and
                  Yuexian Hou and
                  Ming Zhou and
                  Dawei Song},
  editor       = {Hanna M. Wallach and
                  Hugo Larochelle and
                  Alina Beygelzimer and
                  Florence d'Alch{\'{e}}{-}Buc and
                  Emily B. Fox and
                  Roman Garnett},
  title        = {A Tensorized Transformer for Language Modeling},
  booktitle    = {Advances in Neural Information Processing Systems 32: Annual Conference
                  on Neural Information Processing Systems 2019, NeurIPS 2019, December
                  8-14, 2019, Vancouver, BC, Canada},
  pages        = {2229--2239},
  year         = {2019},
  bibsource    = {dblp computer science bibliography, https://dblp.org}
}

@article{tpa2025,
  author       = {Yifan Zhang and
                  Yifeng Liu and
                  Huizhuo Yuan and
                  Zhen Qin and
                  Yang Yuan and
                  Quanquan Gu and
                  Andrew Chi{-}Chih Yao},
  title        = {Tensor Product Attention Is All You Need},
  journal      = {CoRR},
  volume       = {abs/2501.06425},
  year         = {2025},
  url          = {https://doi.org/10.48550/arXiv.2501.06425},
  doi          = {10.48550/ARXIV.2501.06425},
  eprinttype    = {arXiv},
  eprint       = {2501.06425},
  bibsource    = {dblp computer science bibliography, https://dblp.org}
}

@article{chung2005directed,
  title     = {Laplacians and the Cheeger inequality for directed graphs},
  author    = {Chung, Fan},
  journal   = {Annals of Combinatorics},
  volume    = {9},
  number    = {1},
  pages     = {1--19},
  year      = {2005},
  publisher = {Springer}
}

@article{veerman2020primer,
  title   = {A Primer on Laplacian Dynamics in Directed Graphs},
  author  = {Veerman, J. J. P. and Lyons, Robert},
  journal = {Nonlinear Phenomena in Complex Systems},
  volume  = {23},
  number  = {2},
  pages   = {196--206},
  year    = {2020},
  doi     = {10.33581/1561-4085-2020-23-2-196-206}
}

@article{lin2013parallel,
  author       = {Binbin Lin and
                  Xiaofei He and
                  Chiyuan Zhang and
                  Ming Ji},
  title        = {Parallel vector field embedding},
  journal      = {J. Mach. Learn. Res.},
  volume       = {14},
  number       = {1},
  pages        = {2945--2977},
  year         = {2013},
  url          = {https://dl.acm.org/doi/10.5555/2567709.2567755},
  doi          = {10.5555/2567709.2567755},
  bibsource    = {dblp computer science bibliography, https://dblp.org}
}

@inproceedings{tong2025neuralode,
  author       = {Anh Tong and
                  Thanh Nguyen{-}Tang and
                  Dongeun Lee and
                  Duc Nguyen and
                  Toan M. Tran and
                  David Leo Wright Hall and
                  Cheongwoong Kang and
                  Jaesik Choi},
  title        = {Neural {ODE} Transformers: Analyzing Internal Dynamics and Adaptive
                  Fine-tuning},
  booktitle    = {The Thirteenth International Conference on Learning Representations,
                  {ICLR} 2025, Singapore, April 24-28, 2025},
  publisher    = {OpenReview.net},
  year         = {2025},
  bibsource    = {dblp computer science bibliography, https://dblp.org}
}

@article{kan2025ottransformer,
  author       = {Kelvin Kan and
                  Xingjian Li and
                  Stanley J. Osher},
  title        = {OT-Transformer: {A} Continuous-time Transformer Architecture with
                  Optimal Transport Regularization},
  journal      = {CoRR},
  volume       = {abs/2501.18793},
  year         = {2025},
  url          = {https://doi.org/10.48550/arXiv.2501.18793},
  doi          = {10.48550/ARXIV.2501.18793},
  eprinttype    = {arXiv},
  eprint       = {2501.18793},
  bibsource    = {dblp computer science bibliography, https://dblp.org}
}

@inproceedings{bodnar2022neuralsheafdiffusion,
  author       = {Cristian Bodnar and
                  Francesco Di Giovanni and
                  Benjamin Paul Chamberlain and
                  Pietro Li{\`o} and
                  Michael M. Bronstein},
  title        = {Neural Sheaf Diffusion: {A} Topological Perspective on Heterophily
                  and Oversmoothing in GNNs},
  booktitle    = {Advances in Neural Information Processing Systems 35: Annual Conference
                  on Neural Information Processing Systems 2022, NeurIPS 2022, New Orleans,
                  LA, USA, November 28 - December 9, 2022},
  pages        = {18527--18541},
  year         = {2022},
  url          = {https://openreview.net/forum?id=vbPsD-BhOZ},
  bibsource    = {dblp computer science bibliography, https://dblp.org}
}

@inproceedings{barbero2022connectionlaplacians,
  author       = {Federico Barbero and
                  Cristian Bodnar and
                  Haitz S{\'a}ez de Oc{\'a}riz Borde and
                  Michael M. Bronstein and
                  Petar Velickovic and
                  Pietro Li{\`o}},
  title        = {Sheaf Neural Networks with Connection Laplacians},
  booktitle    = {Proceedings of Topological, Algebraic, and Geometric Learning Workshops 2022},
  series       = {Proceedings of Machine Learning Research},
  volume       = {196},
  pages        = {28--36},
  publisher    = {PMLR},
  year         = {2022},
  url          = {https://proceedings.mlr.press/v196/barbero22a.html},
  bibsource    = {dblp computer science bibliography, https://dblp.org}
}

\end{document}